%% file: langevin_on_manifold.tex
\documentclass{article}
\usepackage[main, final]{neurips_2025}
\usepackage{xr}
\usepackage[utf8]{inputenc}
\usepackage{hyperref}       
\usepackage{url}            
\usepackage{booktabs}       
\usepackage{amsfonts}       
\usepackage{nicefrac}       
\usepackage{microtype}      

\usepackage{graphicx}
\usepackage{subcaption}
\usepackage{booktabs} 
\usepackage{amsmath}
\usepackage{amsthm}
\usepackage{amssymb}
\usepackage{algorithm}
\usepackage{algorithmic}
\usepackage{mathtools}
\usepackage{natbib}
\usepackage{thmtools}
\usepackage{thm-restate}
\usepackage{hyperref}
\usepackage{wrapfig}

\usepackage{xcolor}
\usepackage{color}

\newcommand{\bG}{\text{\boldmath{$G$}}}

\newcommand{\bB}{\text{\boldmath{$B$}}}

\newcommand{\bx}{\boldsymbol{x}}
\newcommand{\barx}{\bar{\boldsymbol{x}}}

\newcommand{\bxi}{\boldsymbol{\xi}}

\newcommand{\ba}{\boldsymbol{a}}

\newcommand{\bu}{\boldsymbol{u}}
\newcommand{\bb}{\boldsymbol{b}}

\newcommand{\by}{\boldsymbol{y}}

\newcommand{\beps}{\boldsymbol{\epsilon}}

\newcommand{\bI}{\boldsymbol{I}}
\newcommand{\bh}{\boldsymbol{h}}

\newcommand{\bA}{\boldsymbol{A}}

\newcommand{\bbR}{\mathbb{R}}

\newcommand{\grad}{\mathrm{grad}}

\newcommand{\bv}{\boldsymbol{v}}

\newcommand{\sW}{\mathsf{W}}

\DeclareMathOperator*{\argmin}{arg\,min}

\newtheorem{assumption}{\textbf{Assumption}}

\newtheorem{lemma}{\textbf{Lemma}}
\newtheorem{theorem}{\textbf{Theorem}}

\newtheorem{remark}{\textbf{Remark}}

\newtheorem{example}{\textbf{Example}}

\newcommand{\mE}{\mathbb{E}}

\newcommand{\Var}{\mathsf{Var}}

\newcommand{\cX}{\mathcal{X}}

\newcommand{\cL}{\mathcal{L}}

\newcommand{\cN}{\mathcal{N}}

\newcommand{\cP}{\mathcal{P}}

\newcommand{\cT}{\mathcal{T}}

\newcommand{\cO}{\mathcal{O}}

\newcommand{\tr}{\mathrm{tr}}
\newcommand{\Exp}{\mathrm{Exp}}

\newcommand{\hbx}{\hat{\bx}}

\makeatother

\usepackage[textsize=tiny]{todonotes}

\title{Continuous-time Riemannian SGD and SVRG Flows on Wasserstein Probabilistic Space}
\usepackage{neurips_2025}
\begin{document}

	\author{Mingyang Yi$^{1}$\thanks{Corresponding to: Mingyang Yi and Bohan Wang}, Bohan Wang$^{2}$\\
	$^{1}$ Renmin University of China\\
	$^{2}$ Alibaba Group\\
	\texttt{yimingyang@ruc.edu.cn} \\
	\texttt{bhwangfy@gmail.com}
}
	\maketitle
	\begin{abstract}
		Recently, optimization on the Riemannian manifold have provided valuable insights to the optimization community. In this regard, extending these methods to to the Wasserstein space is of particular interest, since optimization on Wasserstein space is closely connected to practical sampling processes. Generally, the standard (continuous) optimization method on Wasserstein space is Riemannian gradient flow (i.e., Langevin dynamics when minimizing KL divergence). In this paper, we aim to enrich the family of continuous optimization methods in the Wasserstein space, by extending the gradient flow on it into the stochastic gradient descent (SGD) flow and stochastic variance reduction gradient (SVRG) flow. 
		By leveraging the property of Wasserstein space, we construct stochastic differential equations (SDEs) to approximate the corresponding discrete Euclidean dynamics of the desired Riemannian stochastic methods. Then, we obtain the flows in Wasserstein space by Fokker-Planck equation. Finally, we establish convergence rates of the proposed stochastic flows, which align with those known in the Euclidean setting.  
	\end{abstract}
	\section{Introduction}\label{sec:intro}
	As a valuable extension of Euclidean space optimization, generalizing the parameter space to a Riemannian manifold—such as a matrix manifold \citep{absil2009optimization} or a probability measure space \citep{chewi2023optimization}—has greatly enriched the toolbox of the optimization community. Technically, optimization on a manifold can be derived from Euclidean techniques by defining corresponding Riemannian gradients and transport rules \citep{absil2009optimization}. For example, several widely used gradient-based optimization methods have been generalized to this setting, including gradient descent (GD) \citep{zhang2016first}, stochastic gradient descent (SGD) \citep{bonnabel2013stochastic}, and stochastic variance reduced gradient (SVRG) \citep{zhang2016riemannian}.
	\par
	On the other hand, optimization over probability measure spaces has attracted considerable attention due to its connection to sampling processes \citep{dwivedi2018log}. Interestingly, for a specific type of probability measure space—the second-order Wasserstein space—techniques from Riemannian optimization can be directly applied \citep{chewi2023log}, owing to its manifold-like geometric structure. For example, minimizing the KL divergence via Riemannian gradient flow is equivalent to employing Langevin diffusion \citep{parisi1981correlation}, a standard method for sampling from a target distribution. Therefore, advancing Riemannian optimization in the Wasserstein space may lead to the development of novel sampling techniques. 
	\par
	To this end, we focus on Riemannian SGD \citep{li2017stochastic,hu2019diffusion} and SVRG \citep{orvieto2019continuous} in Wasserstein space, which are standard stochastic optimization methods known for their lower computational complexity in Euclidean settings \citep{bottou2018optimization}. Specifically, we investigate the continuous counterparts (flows) of these two previously unexplored Riemannian methods, \textbf{\emph{as continuous formulations provide a powerful framework for analyzing the properties of optimization algorithms}} \citep{du2019gradient,du2018algorithmic}. In Euclidean space, such continuous stochastic flows are characterized by SDEs \citep{li2017stochastic,hu2019diffusion,orvieto2019continuous}. However, these SDEs do not readily generalize to Riemannian manifolds, as their definitions rely on Brownian motion—a concept that is considerably more complex to define on Riemannian manifolds \citep{ren2024ornstein}.
	\par
	In general, continuous optimization methods are derived by taking the limit of the step size in corresponding discrete optimization dynamics (e.g., transitioning from GD to gradient flow \citep{santambrogio2017euclidean}). Naturally, we seek to apply this approach to discrete Riemannian SGD and SVRG \citep{bonnabel2013stochastic,zhang2016riemannian}. Unfortunately, such an extension is nontrivial for two reasons: (1) the linear structure that underpins the aforementioned continuous dynamics does not necessarily exist on manifolds; (2) describing the randomness inherent in stochastic methods is challenging in a manifold setting. Fortunately, the dynamics of probability measures in Wasserstein space can be equivalently described by the dynamics of random vectors in Euclidean space. By taking the step size to zero, the corresponding discrete dynamics in Euclidean space converge to an SDE, which characterizes the evolution of probability measures in Wasserstein space through the Fokker-Planck (F-P) equation \citep{oksendal2013stochastic}.
	\par
	Through this approach, we successfully establish continuous Riemannian SGD and SVRG flows in the Wasserstein space for minimizing the KL divergence as intended. Notably, the existing MCMC methods—stochastic gradient Langevin dynamics \citep{welling2011bayesian} and stochastic variance reduction Langevin dynamics \citep{zou2018subsampled,chatterji2018theory}—are precisely the discrete counterparts of the two proposed flows. Furthermore, under appropriate regularity conditions, we prove convergence rates for these continuous stochastic flows. Specifically, for non-convex problems, the convergence rates (measured by the first-order stationarity criterion \citep{boumal2019global}) of the Riemannian SGD flow and Riemannian SVRG flow are $\mathcal{O}(1 / \sqrt{T})$ and $\mathcal{O}(N^{2/3} / T)$, respectively, where $N$ denotes the number of component functions. Additionally, under an extra Riemannian Polyak-Łojasiewicz (PL) inequality \citep{karimi2016linear,chewi2022analysis}—equivalent to the log-Sobolev inequality \citep{van2014probability}—the two methods achieve global convergence rates of orders $\mathcal{O}(1 / T)$ and $\mathcal{O}(e^{-\gamma T / N^{2/3}})$, respectively, matching their Euclidean counterparts \citep{orvieto2019continuous}. 
	\section{Related Work}\label{sec:related work}
	\paragraph{Riemannian Optimization.} Unlike continuous methods on Riemannian manifolds \citep{absil2009optimization}, discrete methods have been extensively studied. Examples include Riemannian GD \citep{absil2009optimization,boumal2019global,zhang2016first}, Riemannian Nesterov-type methods \citep{zhang2018towards,liu2017accelerated,liu2019understanding}, Riemannian SGD \citep{bonnabel2013stochastic,patterson2013stochastic}, Riemannian SVRG \citep{zhang2016riemannian}, and some other Algorithms \citep{becigneul2018riemannian,zhang2018r,yi2022accelerating,cho2017riemannian}.
	As established in the literature, the standard approach in Euclidean space for linking discrete dynamics to their continuous counterparts involves taking the limit to the step size in the discrete dynamics, which yields a differential equation \citep{su2014differential,santambrogio2017euclidean,liu2017stein}. However, this derivation does not extend directly to arbitrary Riemannian manifolds. Such extrapolation has only been achieved in specific spaces, such as the Wasserstein space, where the resulting curve satisfies the Fokker–Planck equation \citep{santambrogio2017euclidean}, thereby establishing a connection between the Wasserstein and Euclidean spaces. Nevertheless, only a limited number of discrete optimization methods in the Wasserstein space have been generalized to their continuous counterparts—for example, gradient flow \citep{chewi2022analysis}, Nesterov accelerated flow \citep{wang2022accelerated}, and Newton flow \citep{wang2020information}. Unfortunately, these extrapolation techniques have not been applied to stochastic dynamics. Moreover, their methodologies are restricted to specific algorithms, in contrast to the more general framework proposed in this paper.
	
	\paragraph{Stochastic Sampling.} Standard sampling methods, such as MCMC \citep{karatzas2012brownian}, typically construct (stochastic) dynamics that converge to the target distribution, with convergence measured by a probability distance or divergence. Consequently, sampling can be viewed as an optimization problem in the space of probability measures \citep{chewi2023optimization}. However, existing literature has primarily focused on discrete Langevin dynamics \citep{parisi1981correlation} and their stochastic variants \citep{welling2011bayesian,dubey2016variance,zou2018subsampled,chatterji2018theory,zou2019sampling,kinoshita2022improved}, particularly analyzing their convergence rates under different criteria—such as KL divergence \citep{cheng2018convergence}, R'{e}nyi divergence \citep{chewi2022analysis,balasubramanian2022towards,mousavi2023towards}, or Wasserstein distance \citep{durmus2019high}. By contrast, exploring their connections with continuous Riemannian optimization methods, as undertaken in this paper, has received relatively little attention.        
	
	\section{Preliminaries}\label{sec:Preliminaries}
	supported on $\mathcal{X} = \mathbb{R}^{d}$, where $\mathcal{P}$ is the Wasserstein metric space endowed with the second-order Wasserstein distance \citep{villani2009optimal} (abbreviated as Wasserstein distance), defined as $\sW_{2}^{2}(\pi, \mu) = \inf_{\Pi\in \Gamma(\pi,\mu)}\int \|\bx - \by\|^{2}d\Pi(\bx, \by),$ where $\Gamma(\pi,\mu)$ denotes the set of joint probability measures with marginals $\pi$ and $\mu$, respectively. Notably, the Wasserstein space possesses a geometric structure analogous to that of a Riemannian manifold \citep{chewi2023log}, allowing us to apply Riemannian optimization techniques. We now introduce several key definitions. As in $\mathbb{R}^{d}$, the Wasserstein space (similar to Riemannian manifold) is equipped with an ``inner product'', resulting the Riemannian metric $\langle\cdot, \cdot\rangle_{\pi} : \cT_{\pi}\cP \times \cT_{\pi}\cP \to \mathbb{R}$. Here, $\cT_{\pi}\cP$ is the tangent space of $\cP$ at $\pi$ (see \citep{chewi2023log}, Section 1.3). Using this Riemannian metric, we define the Riemannian (Wasserstein) gradient $\grad F(\pi) \in \cT_{\pi}\cP$ of a function $F: \cP \to \mathbb{R}$ as the unique element satisfying $\lim_{t\to0}\frac{F(\pi_{t}) - F(\pi_{0})}{t} = \langle \grad F(\pi_{0}), \bv_{0}\rangle_{\pi_{0}},$ for every curve $\pi_{t}$ in $\cP$ with tangent vector $\bv_{0}\in\cT_{\pi_{0}}$ on $\pi_{0}$. The transportation of $\pi_{t}\in\cP$ in $\cT_{\pi}\cP$ is determined by the exponential map $\Exp_{\pi}:\cT_{\pi}\cP\to\cP$. Then, the discrete Riemannian GD  $\{\pi_{n}\}$ with learning rate $\eta$ is defined as    
	\begin{equation}\label{eq:riemannian gd}
		\small
		\pi_{n + 1} = \Exp_{\pi_{n}}[-\eta \grad F(\pi_{n})],
	\end{equation} 
	to minimize $F(\pi)$. We refer readers for more details about this dynamics to \citep{absil2009optimization,boumal2019global,zhang2016first,chewi2023log}. An illustration of Riemannian GD is in Figure \ref{fig:RGD}.
	
	\begin{wrapfigure}{r}{4cm}
		\centering
				\vspace{-0.3in}
		\includegraphics[scale=0.21]{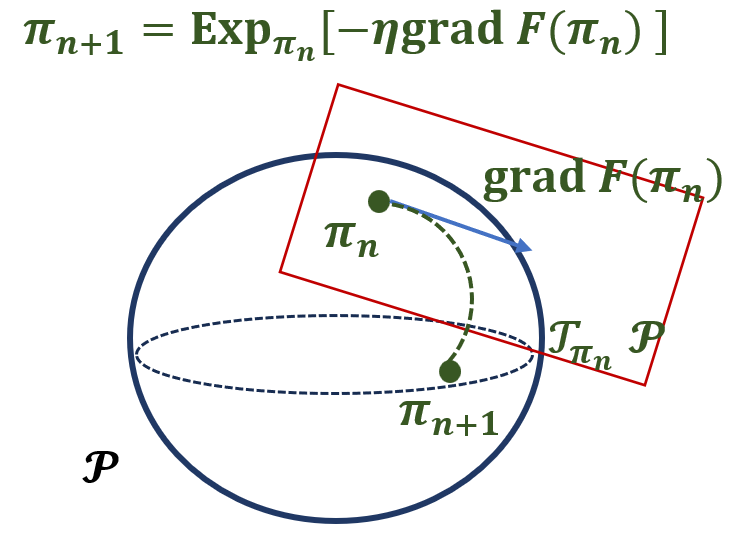}
		\caption{Riemannian GD}
		\label{fig:RGD}
	\end{wrapfigure}
	
	\par
	Next, we instantiate these definitions in the context of the Wasserstein space. First, a curve $\pi_t$ in $\mathcal{P}$ is characterized by the continuity equation (also known as the Fokker–Planck equation) \citep{oksendal2013stochastic}:
	\begin{equation}\label{eq:fp equation}
		\frac{\partial \pi_t}{\partial t} + \nabla \cdot (\pi_t \bv_t) = 0,
	\end{equation}
	where $\pi_t \in \mathcal{P}$ and $\bv_t \in \mathcal{T}_{\pi_t}\mathcal{P}$ \citep{chewi2023log} is a vector field mapping $\mathcal{X} \to \mathcal{X}$. The Fokker–Planck equation implies that $\pi_t$ describes the probability distribution of the stochastic ordinary differential equation (SODE) $d\bx_t = \bv_t(\bx_t) dt$, where the randomness originates from the initial condition $\bx_0$. Consequently, the tangent vector to the curve $\pi_t$, i.e., the direction of curve is given by $\bv_t$ in the F-P equation \eqref{eq:fp equation}. Besides, for $\bu, \bv \in\cT_{\pi}\cP$ the  Riemannian metric \citep{chewi2023optimization} in Wasserstein space is $\langle\bu, \bv\rangle_{\pi} = \int \langle\bu, \bv\rangle d\pi,$ 
	and $\|\bu\|_{\pi}^{2} = \langle\bu, \bu\rangle_{\pi}$, where $\langle\cdot,\cdot\rangle$ is the inner product in Euclidean space. Finally, the exponential map in Wasserstein space is
	\begin{equation}\label{eq:exp map in manifold}
		\small
		\Exp_{\pi}[\bv] = (\bv + \mathrm{id})_{\#\pi},
	\end{equation}
	where $(\bv + \mathrm{id})_{\#\pi}$ is the probability distribution of random variable $\bx + \bv(\bx)$ with $\bx\sim \pi$ (\citep{chewi2023log}, page 44). In this paper, we mainly explore minimizing the KL divergence \citep{van2014probability} to a target probability measure $\mu$
	\begin{equation}\label{eq:minimization of kl divergence}
		\small
		\min_{\pi\in\cP} F(\pi) = \min_{\pi\in\cP} D_{KL}(\pi\parallel \mu) = \min_{\pi\in\cP}\int \log{\frac{d\pi}{d\mu}}d\pi. 
	\end{equation} 
	Here, the target distribution $\mu$ is assumed to satisfy $\mu \propto \exp(-V(\bx))$ for some potential function $V(\bx)$ \citep{chewi2023optimization}. For the minimization problem \eqref{eq:minimization of kl divergence}, global convergence results have been established under specific regularity conditions, such as the log-Sobolev inequality \citep{chewi2023log,van2014probability}. For the KL divergence $F(\pi) = D_{\mathrm{KL}}(\pi \parallel \mu)$, this inequality takes the form
	\begin{equation}\label{eq:log Sobolev inequality}
		\begin{aligned}
			D_{\mathrm{KL}}(\pi \parallel \mu) \leq \frac{1}{2\gamma} \int \left| \nabla \log \frac{d\pi}{d\mu} \right|^2 d\pi = \frac{1}{2\gamma} \left| \grad_{\pi} D_{\mathrm{KL}}(\pi \parallel \mu) \right|^2_{\pi},
		\end{aligned}
	\end{equation}
	for some $\gamma > 0$ and all $\pi$, where the equality follows from Proposition \ref{pro:riemannian gradient}. In what follows, we write $\grad_{\pi} D_{\mathrm{KL}}(\pi \parallel \mu)$ as $\grad D_{\mathrm{KL}}(\pi \parallel \mu)$ to simplify notation. In fact, the log-Sobolev inequality in the Wasserstein space generalizes the PL inequality \citep{karimi2016linear,chewi2022analysis}, thereby guaranteeing global convergence of optimization methods. Further details can be found in Section \ref{sec:More than Riemannian PL inequality} of the Appendix. 
	\par
	In this paper, we need the following lemma from \citep{maoutsa2020interacting,song2020score}, which connects the SODE and SDE. 
	\begin{lemma}\citep{maoutsa2020interacting,song2020score}\label{lem:equivalence}
		The SDE $d\bx_{t} = \bb(\bx_{t}, t)dt + \bG(\bx_{t}, t)dW_{t},$ 
		has the same density with SODE 
		\begin{equation}
			\small
			\begin{aligned}
				d\bx_{t} = \bb(\bx_{t}, t) - \frac{1}{2}\nabla\cdot\left[\bG(\bx_{t}, t)\bG^{\top}(\bx_{t}, t)\right] - \frac{1}{2}\bG(\bx_{t}, t)\bG^{\top}(\bx_{t}, t)\nabla\log{\pi_{t}(\bx_{t})}dt,
			\end{aligned}
		\end{equation}
		where $\pi_{t}$ \footnote{We simplify $\log{(d\pi_{t} / d\bx)(\bx_{t})}$ as $\log{\pi_{t}(\bx_{t})}$ if there is no obfuscation in sequel.} is the corresponded probability measure of $\bx_{t}$. 
	\end{lemma}
	 As shown by this lemma and \eqref{eq:fp equation}, the direction of $\pi_t$ can be directly linked to a SDE. For further details on these preliminaries, we refer readers to \citep{absil2009optimization, boumal2019global, chewi2023log}.
	\section{Riemannian Gradient Flow}\label{sec:riemannian gradient flow}
	In this section, we investigate the continuous gradient flow for minimizing the KL divergence in the Wasserstein space through the framework of Riemannian manifold optimization. Although this continuous optimization method has been studied previously \citep{liu2017accelerated,santambrogio2017euclidean,chewi2023log}, \textbf{a systematic analysis from the perspective of manifold optimization remains lacking}. We demonstrate that our approach provides valuable insights for the subsequent development of Riemannian SGD and SVRG flows.  
	\subsection{Constructing Riemannian Gradient Flow}
	We first calculate the Riemannian gradient of KL divergence (proved in Appendix \ref{app:proofs in riemannian gradient flow}).
	\begin{restatable}{proposition}{riemanniangradient}
		The Riemannian gradient of $F(\pi) = D_{KL}(\pi \parallel \mu)$ in Wasserstein space is 
		\begin{equation}\label{eq:riemannian gradient of KL}
			\small
			\grad F(\pi) = \grad D_{KL}(\pi \parallel \mu) = \nabla\log{\frac{d\pi}{d\mu}}. 
		\end{equation}
	\end{restatable}
	With the defined Riemannian gradient of the KL divergence and exponential map \eqref{eq:exp map in manifold}, we can implement the discrete Riemannian gradient descent as in \eqref{eq:riemannian gd}. 
	\par	
	As noted in Section \ref{sec:intro}, we aim to establish a correspondence between discrete dynamics and their continuous counterparts. In Euclidean space, the GD dynamics $(\bx_{n + 1} - \bx_{n}) / \eta = -\nabla F(\bx_{n})$ leads to an ODE (gradient flow) $d\bx_{t} = -\nabla F(\bx_{t})dt$ in the limit as $\eta \to 0$. However, this approach does not directly extend to the manifold setting, since the Riemannian GD dynamics $\pi_n$ in \eqref{eq:riemannian gd} does not induce a linear structure. Fortunately, the probability measure $\pi_n$ corresponds to random vectors $\bx_n \in \mathbb{R}^d$ such that $\bx_n \sim \pi_n$ and $\bx_{n + 1} = \bx_n - \eta \nabla \log \frac{d\pi_n}{d\mu}(\bx_n)$. Consequently, the dynamics induced by Riemannian GD \eqref{eq:riemannian gd} can be used to construct a differential equation in the limit $\eta \to 0$. The corresponding Fokker–Planck equation for this differential equation in the Wasserstein space then yields the gradient flow. The conceptual framework is illustrated in Figure \ref{fig:idea}, and the formal result is stated in the following proposition.
	\begin{wrapfigure}{r}{6cm}
		\includegraphics[scale=0.25]{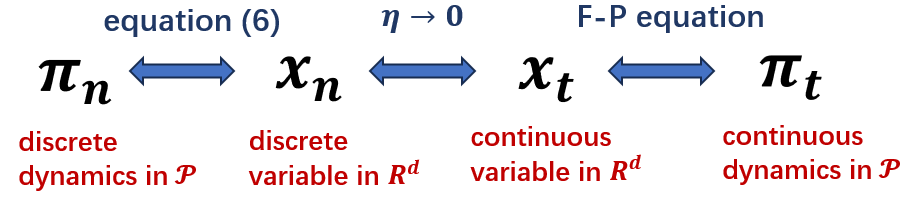}
		\caption{Our idea to bridge the discrete dynamics to its continuous counterparts.}
		\vspace{-0.2in}
		\label{fig:idea}
	\end{wrapfigure} 
	\begin{assumption}\label{ass:lip continuity of log mu for gd}
		For probability measure $\mu$, the $\log{\mu}$ and $\nabla\log{\mu}$ are all Lipschitz continuous with coefficient $L_{1}$ and $L_{2}$, respectively. \footnote{Since $\mu\propto\exp(-V(\bx))$, the condition means $V(\bx)$ and $\nabla V(\bx)$ are all Lipschitz continuous.} 
	\end{assumption}
	\begin{restatable}{proposition}{riemanniangradientflow}\label{pro:riemannian gradient}
		Under Assumption \ref{ass:lip continuity of log mu for gd} and $1\leq n \leq \cO(\lfloor1 / \eta\rfloor)$, the discrete Riemannian GD \eqref{eq:riemannian gd} approximates continuous Riemannian gradient flow $\pi_{t}$ 
		\begin{equation}\label{eq:lagevin FP}
			\small
			\frac{\partial}{\partial{t}}\pi_{t} = \nabla\cdot(\pi_{t}\grad D_{KL}(\pi_{t}\parallel \mu)) = \nabla\cdot\left(\pi_{t}\nabla\log{\frac{d\pi_{t}}{d\mu}}\right),
		\end{equation}
		by $\mE[\|\bx_{n} - \hbx_{n\eta}\|^{2}] \leq \cO(\eta)$, where $\bx_{n}\sim \pi_{n}$ in \eqref{eq:riemannian gd}, $\hbx_{n\eta}\sim \pi_{n\eta}$ in \eqref{eq:lagevin FP}, for $\bx_{0} = \hbx_{0}$. 
	\end{restatable}
	This proposition shows that the discrete Riemannian GD approximates the flow \eqref{eq:lagevin FP} for any finite product $n\eta$ in the limit as $\eta \to 0$. Therefore, the flow \eqref{eq:lagevin FP} indeed corresponds to the Riemannian gradient flow. Moreover, by combining Lemma \ref{lem:equivalence} with the F-P equation \eqref{eq:fp equation}, the measure $\pi_t$ in \eqref{eq:lagevin FP} is the distribution of Langevin dynamics, which serves as a standard stochastic sampling algorithm \citep{parisi1981correlation} given by $d\bx_t = \nabla \log \mu(\bx_t) dt + \sqrt{2} dW_t$. This connection is also discussed in \citep{chewi2023optimization,santambrogio2017euclidean,liu2019understanding}. 
	\begin{remark}
		Notably, existing literature \citep{santambrogio2017euclidean} has also established that \eqref{eq:lagevin FP} corresponds to the Riemannian gradient flow in the Wasserstein space for minimizing the KL divergence. \textbf{However, their results are not derived from a limiting procedure on the learning rate in discrete Riemannian optimization methods, unlike our approach.} The authors \citep{santambrogio2017euclidean} obtain the curve $\pi_t$ by solving the variational problem $\pi_{n + 1} = \argmin_{\pi} \left[ D_{\mathrm{KL}}(\pi \parallel \mu) + W_2^2(\pi, \pi_n) / 2\eta \right]$ in the limit $\eta \to 0$. Unfortunately, unlike our method, this approach cannot be extended to the stochastic optimization setting in Section \ref{sec:riemannian sgd flow}, as the aforementioned minimization problem does not readily incorporate stochastic gradients.
	\end{remark}
	\subsection{Convergence of Riemannian GD Flow}
	In this section, we prove the convergence of the Riemannian gradient flow \eqref{eq:lagevin FP}. For non-convex problems in Euclidean space, convergence is typically measured by the first-order stationarity criterion \citep{ghadimi2013stochastic,boumal2019global,yi2022characterization}. Accordingly, in the Wasserstein space, we analogously analyze the convergence rate of the Riemannian gradient norm $|\grad D_{\mathrm{KL}}(\pi_t \parallel \mu)|{\pi_t}^2$ \footnote{$|\grad D{\mathrm{KL}}(\pi_t \parallel \mu)|_{\pi_t}^2 \to 0$ does not imply $D(\pi_t, \mu) \to 0$ for other probability distances or divergences $D(\cdot, \cdot)$, such as the total variation distance. Further details can be found in \citet{balasubramanian2022towards}.}, following the approach of \citet{boumal2019global,balasubramanian2022towards}. 
	Moreover, the PL inequality \citep{karimi2016linear} ensures global convergence in Euclidean space; for instance, under this condition, the gradient flow converges exponentially to a global minimum \citep{karimi2016linear}. As discussed in Section \ref{sec:Preliminaries}, in the Wasserstein space, the Riemannian PL inequality is generalized by the log-Sobolev inequality \citep{van2014probability}, which likewise guarantees an exponential global convergence rate for the Riemannian gradient flow, as stated in the following theorem.
	\begin{restatable}{theorem}{convergencergd}[\citep{chewi2023log}]\label{thm:convergence riemannian gd}
		Let $\pi_{t}$ follows the Riemannian gradient flow \eqref{eq:lagevin FP}, then for any $T > 0$, we have 
		\begin{equation}
			\small
			\frac{1}{T}\int_{0}^{T}\|\grad D_{KL}(\pi_{t} \parallel \mu)\|_{\pi_{t}}^{2}dt \leq \frac{D_{KL}(\pi_{0} \parallel \mu)}{T}. 
		\end{equation}
		Moreover, if the log-Sobolev inequality \eqref{eq:log Sobolev inequality} (Riemannian PL inequality) is satisfied for $\mu$, then 
		\begin{equation}
			\small
			D_{KL}(\pi_{t} \parallel \mu) \leq e^{-2\gamma t}D_{KL}(\pi_{0} \parallel \mu).
		\end{equation}
	\end{restatable}
	We prove it in Appendix \ref{app:proofs in riemannian gradient flow}. The convergence rate of Riemannian gradient flow is also proved in \citep{chewi2023log} and it matches the results in Euclidean setting \citep{su2014differential,karimi2016linear} as expected. We prove this theorem here to illustrate the criteria of convergence rates under different conditions. 
	
	\section{Riemannian Stochastic Gradient Flow}\label{sec:riemannian sgd flow}
	In practice, SGD is preferred over GD due to its lower computational complexity. However, unlike in Euclidean space \citep{hu2019diffusion,li2017stochastic}, the continuous SGD flow in Wasserstein space has not been explored yet. Next, our goal is to generalize the Riemannian gradient flow \eqref{eq:lagevin FP} to the Riemannian SGD flow.  
	\subsection{Constructing Riemannian Stochastic Flow}
	The stochastic algorithm is developed to minimize the stochastic optimization problem such that 
	\begin{equation}\label{eq:stochastic objective}
		\small
		\min_{\pi} \mE_{\xi}\left[f_{\xi}(\pi)\right] = \min_{\pi}\mE_{\xi}[D_{KL}(\pi\parallel \mu_{\xi})],  
	\end{equation} 
	where the expectation is taken over $\xi$ parameterizes a set of probability measures $\mu_{\xi}$ \footnote{$\mu_{\xi}(\bx)$ is assumed to be $\mu_{\xi}(\bx)\propto \exp(-V_{\xi}(\bx))$.}. For objective \eqref{eq:stochastic objective}, we can get the optima of it by the following proposition proved in Appendix \ref{app:proofs in section riemannian sgd flow}. 
	\begin{restatable}{proposition}{stochasticoptima}\label{eq:stochastic optima}
		The global optima of problem \eqref{eq:stochastic objective} is $\mu \propto \exp\left(\mE_{\xi}\left[\log{\mu_{\xi}}\right]\right)$
	\end{restatable}
	Since we can explicitly get the optima $\mu$ defined in Proposition \eqref{eq:stochastic optima}, the goal of Riemannian SGD flow should be moving towards it. Thus, the target distribution becomes $\mu$ Proposition \eqref{eq:stochastic optima} in the sequel. 
	\par
	To solve the problem \eqref{eq:stochastic objective}, one may use the standard discrete method, Riemannian SGD \citep{bonnabel2013stochastic,yi2022accelerating} as outlined in Algorithm \ref{alg:discrete riemannian sgd}.
	\begin{algorithm}[t!]
		\caption{Discrete Riemannian SGD}
		\label{alg:discrete riemannian sgd}
		\textbf{Input:} Exponential map $\Exp$, initialized $\pi_{0}$, learning rate $\eta$, steps $M$. \indent  
		\begin{algorithmic}[1]
			\FOR    {$n = 0, \cdots ,M - 1$}
			\STATE  {Sample $\xi_{n}\sim\xi$ independent with $\pi_{n}$;}
			\STATE  {Update $\pi_{n + 1} = \Exp_{\pi_{n}}\left[-\eta \grad D_{KL}(\pi\parallel \mu_{\xi_{n}})\right]$;} 
			\ENDFOR
			\STATE {\textbf{Return:} $\pi_{M}$.}
		\end{algorithmic}
	\end{algorithm}
	By analogy with our derivation of the Riemannian GD flow in Section \ref{sec:riemannian gradient flow}, the Riemannian SGD flow is naturally posited as its continuous counterpart. To derive this flow, we first construct the Euclidean-space dynamics ${\bx_n}$ of Algorithm \ref{alg:discrete riemannian sgd}. Then, following a procedure similar to that in Proposition \ref{pro:riemannian gradient}, we approximate the dynamics of $\bx_n$ by a continuous SDE. The corresponding F-P equation then yields the continuous Riemannian SGD flow. Below is the result. 
	\begin{assumption}\label{ass:stochastic continuous}
		For any $\xi$ and probability measure $\mu_{\xi}$, $\log{\mu_{\xi}}$ and $\nabla\log{\mu_{\xi}}$ are Lipschitz continuous with coefficient $L_{1}$ and $L_{2}$ respectively. \footnote{As $\mu_{\xi}\propto \exp(-V_{\xi})$, the assumption implies $V_{\xi}$ and its gradient are Lipschitz continuous.} 
	\end{assumption}
	\begin{restatable}{proposition}{riemanniansgdflow}\label{pro:equivalence of SGD}
		Under Assumption \ref{ass:stochastic continuous}, let $f_{\xi}(\pi) = D_{KL}(\pi\parallel \mu_{\xi})$, the discrete Riemannian SGD Algorithm \ref{alg:discrete riemannian sgd} with $1\leq n \leq \cO(\lfloor1 / \eta\rfloor)$ approximates the continuous Riemannian stochastic gradient flow  
		\begin{equation}\label{eq:sgd flow on riemannian}
			\small
			\frac{\partial}{\partial{t}}\pi_{t} = \nabla\cdot\left[\pi_{t}\left(\nabla\log{\frac{d\pi_{t}}{d\mu}} - \frac{\eta}{2}\nabla\cdot \Sigma_{\rm SGD} - \frac{\eta}{2}\Sigma_{\rm SGD}\nabla\log{\pi_{t}}\right)\right],
		\end{equation}
		by $\mE[\|\bx_{n} - \hbx_{n\eta}\|^{2}] \leq \cO(\eta)$, where $\bx_{n}\sim\pi_{n}$ in Algorithm \ref{alg:discrete riemannian sgd}, $\hbx_{n\eta}\sim \pi_{n\eta}$ in \eqref{eq:sgd flow on riemannian}, for $\bx_{0} = \hbx_{0}$. Here 
		\begin{equation}
			\small
			\begin{aligned}
				\Sigma_{\rm SGD}(\bx) = \mE_{\xi}[\left(\nabla\log{\mu_{\xi}}(\bx) - \nabla\mE_{\xi}\left[\log{\mu_{\xi}}(\bx)\right]\right)\left(\nabla\log{\mu_{\xi}}(\bx) - \nabla\mE_{\xi}\left[\log{\mu_{\xi}}(\bx)\right]\right)^{\top}].
			\end{aligned}
		\end{equation} 
	\end{restatable}
	As can be seen, similar to Proposition \ref{pro:riemannian gradient}, we approximate the discrete Riemannian SGD with continuous flow \eqref{eq:sgd flow on riemannian}, so that it is Riemannian SGD flow as desired. 
	\par
	One may observe that the Riemannian ``stochastic'' gradient flow \eqref{eq:sgd flow on riemannian} is a deterministic curve in the Wasserstein space $\mathcal{P}$. This is not inconsistent with the randomness in the index $\xi_n$ for $\pi_n$ in discrete Riemannian SGD. This is because the stochasticity introduced by $\xi_n$ when obtaining $\pi_n$ (with $\bx_n \sim \pi_n$) in Algorithm \ref{alg:discrete riemannian sgd} is implicitly captured in the corresponding $\bx_n$. In other words, the randomnesses from the random vector $\bx_n$ itself and the random index $\xi_n$ are fully incorporated into $\bx_n$ and manifested in its continuous approximation $\hat{\bx}_{n\eta} \sim \pi_{n\eta}$, where $\pi_t \in \mathcal{P}$ is the deterministic curve \eqref{eq:sgd flow on riemannian} in the Wasserstein space. Thus, all randomness in the Riemannian SGD is accounted in this formulation.
	\par
	Besides, we can find that implementing the discrete Riemannian SGD is non-trivial, since it requires Riemannian gradient $\nabla{\log{(d\pi_{t}/d\mu})}$. However, during proving Proposition \ref{pro:equivalence of SGD}, we show the discrete stochastic gradient Langevin dynamics (SGLD) \citep{welling2011bayesian}
	\begin{equation}\label{eq:discrete sgld}
		\small
		\bx_{n + 1} = \bx_{n} + \eta\nabla\log{\mu_{\xi_{n}}(\bx_{n})} + \sqrt{2\eta}\beps_{n}
	\end{equation}
	approximates \footnote{The approximation is verified by noting $\mE_{\pi}[\langle\nabla f, \nabla\log{\pi}\rangle] = -\mE_{\pi}[\Delta f]$ for continuous test function $f$, and combining Taylor's expansion, please check Appendix \ref{app:proofs in section riemannian sgd flow} for more details.} the corresponded dynamics of $\{\bx_{n}\}$ in discrete Riemannian SGD (Algorithm \ref{alg:discrete riemannian sgd})
	\begin{equation}\label{eq:discrete Riemannian sgd}
		\small
		\bx_{n + 1} = \bx_{n} + \eta \nabla\log\frac{d\mu_{\xi_{n}}}{d\pi_{n}}(\bx_{n}). 
	\end{equation}
	Thus, in practice, discrete SGLD can be implemented to approximate discrete Riemannian SGD. Furthermore, based on this approximation and Proposition \ref{pro:equivalence of SGD}, the continuous counterpart $\bx_t$ of \eqref{eq:discrete Riemannian sgd}, as governed by the Riemannian SGD flow \eqref{eq:sgd flow on riemannian} (Lemma \ref{lem:equivalence}), satisfies the SDE
	\begin{equation}\label{eq:sde of riemannian sgd}
		\small
		d\bx_{t} = \nabla\log{\frac{d\mu}{d\pi_{t}}(\bx_{t})}dt + \sqrt{\eta}\Sigma_{\rm SGD}^{\frac{1}{2}}(\bx_{t})dW_{t},
	\end{equation}
	is also the continuous limit of the discrete SGLD \eqref{eq:discrete sgld}. This establishes a connection between the Riemannian SGD flow and discrete SGLD—that is, the Riemannian SGD flow in the Wasserstein space corresponds precisely to continuous SGLD. \textbf{Consequently, our Riemannian SGD flow provides a powerful framework for analyzing discrete SGLD or Riemannian SGD.} For example, combining Proposition \ref{pro:equivalence of SGD} and Theorem \ref{thm:convergence riemannian sgd} implies the convergence rate of discrete Riemannian SGD.
	\subsection{Convergence of Riemannian SGD Flow}\label{sec:convergence rate of riemannian sgd flow}
	Next, we examine the convergence rate of Riemannian SGD flow. As in Section \ref{sec:riemannian gradient flow}, our analyses are respectively conducted with/without log-Sobolev inequality. 
	\begin{restatable}{theorem}{convergenceofriemanniansgd}\label{thm:convergence riemannian sgd}
		Let $\pi_{t}$ follows the Riemannian SGD flow \eqref{eq:sgd flow on riemannian} and $\mu$ defined in Proposition \eqref{eq:stochastic optima}. Under Assumption \ref{ass:stochastic continuous}, if $T \geq \frac{64L_{1}^{4}D_{KL}(\pi_{0}\parallel \mu)}{4dL_{1}^{2}L_{2} + (d + 1)^{2}L_{2}^{2}}$, then by taking $\eta = \sqrt{\frac{D_{KL}(\pi_{0}\parallel \mu)}{T(4dL_{1}^{2}L_{2} + (d + 1)^{2}L_{2}^{2})}}$, we have 
		\begin{equation}
			\small 
			\begin{aligned}
				\frac{1}{\eta T}\int_{0}^{\eta T}\|\grad D_{KL}(\pi_{t}\parallel \mu)\|^{2}_{\pi_{t}}dt \leq \frac{4D_{KL}(\pi_{0}\parallel \mu)}{\eta T} = \cO\left(\frac{1}{\sqrt{T}}\right).
			\end{aligned}
		\end{equation}
		Besides that, if \eqref{eq:log Sobolev inequality} is satisfied for $\mu$, $\eta = 1 / \gamma T^{\alpha}$ with $0 < \alpha < 1$, and $T \geq \left(8L_{1}^{2} / \gamma\right)^{1/\alpha}$, then  
		\begin{equation}
			\small
			\begin{aligned}
				D_{KL}(\pi_{\eta T} \parallel \mu) \leq \frac{1}{\gamma T^{\alpha}}\left[4dL_{1}^{2}L_{2} + (d + 1)^{2}L_{2}^{2}\right]  = \cO\left(\frac{1}{T^\alpha}\right).
			\end{aligned}
		\end{equation}
	\end{restatable}
	The proof of this theorem is provided in Appendix \ref{app:convergence sgd}. Notably, the presence of $\eta$ in the SDE leads to a convergence rate of order $\mathcal{O}(1 / \sqrt{T})$ for the Riemannian SGD flow\footnote{We emphasize that no constraints are imposed on the learning rate $\eta$ in Theorem \ref{thm:convergence riemannian gd}; hence, the convergence rate of Riemannian GD remains as stated in Theorem \ref{thm:convergence riemannian gd}.}. Under the log-Sobolev inequality, a global convergence rate of $\mathcal{O}(1 / T)$ can be established (by taking $\alpha \to 1$). Therefore, the convergence rates proved in Theorem \ref{thm:convergence riemannian sgd} align with those of the continuous SGD flow in Euclidean space \citep{ghadimi2013stochastic,orvieto2019continuous}, as demonstrated in Appendix \ref{app:convergence sgd}.
	\begin{remark}
		During the proof to Theorem \ref{thm:convergence riemannian sgd}, we assume $\log{\mu_{\xi}}$ is Lipschitz continuous. The assumption can be relaxed as $\mE_{\xi, \bx_{t}}[\|\nabla\log{\mu(\bx_{t})} - \mE_{\xi}[\nabla\log{\mu(\bx_{t})}]\|^{2}] \leq \sigma^{2}$, for constant $\sigma$ and $\bx_{t}\sim \pi_{t}$ in \eqref{eq:sgd flow on riemannian}, which is the standard ``bounded variance'' assumption in optimization on Euclidean space \citep{ghadimi2013stochastic}. 
	\end{remark}
	
	\par
	In the remainder of this section, we further demonstrate the tightness of the derived convergence rate for the Riemannian SGD flow through the following example. Although this example does not satisfy the Lipschitz continuity condition on $\log \mu_{\xi}$ in Assumption \ref{ass:stochastic continuous}, it does satisfy the condition of bounded variance discussed in the preceding remark. Consequently, the results in Theorem \ref{thm:convergence riemannian sgd} remain valid.
	\begin{example}\label{example:gaussian}
		Let $\mu_{\bxi}\sim\cN(\bxi, \bI)$, with $\bxi\in\{\bxi_{1},\cdots, \bxi_{N}\}$, $\max_{1\leq j\leq N}\|\bxi_{j}\| \leq C$ for a constant $C$. 
	\end{example}
	Due to Proposition \eqref{eq:stochastic optima}, we have $\mu\sim \cN(\bar{\bxi}, \bI)$ with $\mE[\bxi] = \bar{\bxi} = \sum_{j}\bxi_{j} / N$, which is the target measure of Riemannian SGD flow. Then, we have $\Sigma_{\rm SGD} = \frac{1}{N}\sum_{j=1}^{N}(\bxi_{j} - \mE[\bxi])(\bxi_{j} - \mE[\bxi])^{\top} = \Var(\bxi)$. 
    \par
    The assumptions (including log-Sobolev inequality) in Theorem \ref{thm:convergence riemannian sgd} are all satisfied (see Lemma \ref{lemma:regularity bound} in Appendix). Then, the corresponding SDE of Riemannian SGD flow \eqref{eq:sgd flow on riemannian} is
	\begin{equation}
		\small
		d\bx_{t} = -(\bx_{t} - \bar{\bxi})dt + (\sqrt{\eta}\Var^{\frac{1}{2}}(\bxi), \sqrt{2}\bI)dW_{t},
	\end{equation}
	which has the following closed-form solution $\bx_{t} = \bar{\bxi} + e^{-t}(\bx_{0} - \bar{\bxi}) + e^{-t}\sqrt{\frac{\eta}{2}}\Var^{\frac{1}{2}}(\bxi)W_{e^{2t} - 1}^{(1)} + e^{-t}W_{e^{2t} - 1}^{(2)},$
%
	where $W^{(1)}_{t}$ and $W^{(2)}_{t}$ are standard independent Brownian motions. By taking $t = \eta T$, for any $\bx_{0}$, 
	\begin{equation}
		\small
		\bx_{\eta T}\sim \cN\left(\bar{\bxi} + e^{-\eta T}(\bx_{0} - \bar{\bxi}), (1 - e^{-2\eta T})(\eta\Var(\bxi)/2 + \bI)\right). 
	\end{equation}
	Here, $e^{-\eta T} \approx 0$ for large $T$, and $\eta = \cO(1/T^{\alpha})$ with $0 < \alpha < 1$. Then, $\bx_{\eta T}\approx \cN(\bar{\bxi}, \frac{\eta}{2}\Var(\bxi) + \bI)$. The Riemannian gradient $\grad D_{KL}(\pi_{\eta T}\parallel \mu) = [(\bI + \eta\Var(\bxi)/2)^{-1} - \bI](\bx_{\eta T} - \bar{\bxi}),$
	which indicates 
	\begin{equation}\label{eq:example gradient sgd convergence}
		\small
		\begin{aligned}
			\left\|\grad D_{KL}(\pi_{\eta T}\parallel \mu)\right\|_{\pi_{\eta T}}^{2} = \tr\left(\left(\bI + \eta\Var(\bxi) / 2\right)^{-1} - \bI + \eta\Var(\bxi)/2\right) = \cO\left(T^{-\alpha}\right).
		\end{aligned}
	\end{equation}
	On the other hand, the KL divergence between Gaussian measures \citep{petersen2008matrix} $\pi_{\eta T}, \mu$ can be calculated as 
	\begin{equation}\label{eq:example global sgd convergence}
		\small
		\begin{aligned}
			D_{KL}(\pi_{\eta T}\parallel \mu) \approx \frac{1}{2}[\log{|\bI + \eta\Var(\bxi) / 2|} + \eta\tr(\Var(\bxi)) / 2] = \cO\left(T^{-\alpha}\right).
		\end{aligned}
	\end{equation}
	As can be seen, the convergence rates in \eqref{eq:example gradient sgd convergence} and \eqref{eq:example global sgd convergence} are consistent with the proved convergence rates in Theorem \ref{thm:convergence riemannian sgd}, under $\alpha = 1/2$ and $\alpha\to 1$, respectively. Thus, the example \eqref{example:gaussian} indicates the convergence rates in Theorem \ref{thm:convergence riemannian sgd} are sharp. 
	\section{Riemannian Stochastic Variance Reduction Gradient Flow}\label{sec:svrg flow}
	Next, we extend discrete algorithm SVRG \citep{johnson2013accelerating} to its continuous counterpart in Wasserstein space. 
	\subsection{Constructing Riemannian SVRG Flow}\label{sec:Constructing Riemannian SVRG Flow}
	In practice, the objective in \eqref{eq:stochastic objective} often takes the form of a finite sum, where $\xi$ is uniformly distributed over ${\xi_1, \cdots, \xi_N}$, so that the objective becomes $\min_{\pi} F(\pi) = \min_{\pi}\mE_{\xi}[f_{\xi}(\pi)] = \min_{\pi}\frac{1}{N}\sum_{j=1}^{N}f_{\xi_{j}}(\pi)$. 
	Clearly, computing $\grad F(\pi)$ requires $\mathcal{O}(N)$ operations. The convergence rates in Theorems \ref{thm:convergence riemannian gd} and \ref{thm:continuous sgd convergence} imply that achieving $|\grad D_{\mathrm{KL}}(\pi_t \parallel \mu)|_{\pi_t}^2 \leq \epsilon$ for some $t$—that is, reaching an $\epsilon$-stationary point—requires computational complexities of $\mathcal{O}(N \epsilon^{-1})$ and $\mathcal{O}(\epsilon^{-2})$, respectively\footnote{The computational complexity of continuous optimization is evaluated by implementing the corresponding discrete algorithm. Further details are provided in Appendix \ref{app:convergence sgd}}. Therefore, for large $N$, Riemannian SGD flow offers an improvement over Riemannian GD flow.
	\par
	In Euclidean space, considerable efforts have been devoted to further improving computational complexity, with methods such as SVRG \citep{johnson2013accelerating,zou2018subsampled}, SPIDER \citep{fang2018spider,zhang2018r}, and SARAH \citep{nguyen2017sarah}. Among these, the double-loop structure of SVRG represents a core idea shared across several variance-reduced approaches. For this reason, we focus specifically on SVRG in this paper.  
	\par 
	\begin{algorithm}[t!]
		\caption{Discrete Riemannian SVRG}
		\label{alg:discrete riemannian svrg}
		\textbf{Input:} Exponential map $\Exp_{\pi}$, initialized $\pi_{0}$, learning rate $\eta$, epoch $I$, steps $M$ of each epoch.
		\begin{algorithmic}[1]
			\STATE  {Take $\pi_{0}^{0} = \pi_{0}$;}
			\FOR    {$i = 0, \cdots ,I - 1$}
			\STATE  {Compute $\grad F(\pi_{0}^{i}) = \frac{1}{N}\sum_{j=1}^{N}\grad f_{\xi_{j}}(\pi_{0}^{i})$}
			\FOR    {$n = 0, \cdots, M - 1$}
			\STATE  {Uniformly sample $\xi_{n}^{i}\in\{\xi_{1},\cdots, \xi_{N}\}$ independent with $\pi_{n}^{i}$;}
			\STATE  {Update $\pi_{n + 1}^{i} = 
				\Exp_{\pi_{n}^{i}}[-\eta(\grad f_{\xi_{n}^{i}}(\pi_{n}^{i}) - 
				\Gamma_{\pi_{0}^{i}}^{\pi_{n}^{i}}(\grad f_{\xi_{n}^{i}}(\pi_{0}^{i}) - \grad F(\pi_{0}^{i})))]$;} 
			\ENDFOR
			\STATE  {$\pi^{i + 1}_{0} = \pi^{i}_{M}$;}
			\ENDFOR
			\STATE {\textbf{Return:} $\pi_{N}$.}
		\end{algorithmic}
	\end{algorithm}
	As presented in Section \ref{sec:riemannian sgd flow}, we begin with the discrete Riemannian SVRG method \citep{zhang2016riemannian} outlined in Algorithm \ref{alg:discrete riemannian svrg}. Analogous to the Euclidean setting, line 6 of Algorithm \ref{alg:discrete riemannian svrg} employs the tangent vector $\grad f_{\xi_{n}^{i}}(\pi_{n}^{i}) - 
	\Gamma_{\pi_{0}^{i}}^{\pi_{n}^{i}}(\grad f_{\xi_{n}^{i}}(\pi_{0}^{i}) - \grad F(\pi_{0}^{i})))$ which serves as a variance-reduced estimator of the Riemannian gradient $\grad F(\pi_{n}^{i})$. This constitutes the core idea of SVRG-type algorithms. In this context, the mapping $\Gamma_{\pi_{0}^{i}}^{\pi_{n}^{i}}$ denotes parallel transport \citep{absil2009optimization,zhang2016riemannian}. Specifically, for the loss function $f_{\xi}(\pi) = D_{\text{KL}}(\pi \parallel \mu_{\xi})$, the difference of Riemannian gradients $\grad f_{\xi_{n}^{i}}(\pi_{0}^{i}) - \grad F(\pi_{0}^{i})$ lies in the tangent space $\cT_{\pi_{0}^{i}}\cP$, while the exponential map $\Exp_{\pi_{n}^{i}}$ is defined on $\cT_{\pi_{n}^{i}}\cP$. To reconcile this discrepancy, the parallel transport  $\Gamma_{\pi_{0}^{i}}^{\pi_{n}^{i}}: \cT_{\pi_{0}^{i}}\cP\rightarrow \cT_{\pi_{n}^{i}}\cP$ is applied, which maps vectors from $\cT_{\pi_{0}^{i}}\cP$ to $\cT_{\pi_{n}^{i}}\cP$ while preserving the Riemannian metric, i.e.,  $\|\Gamma_{\pi_{0}^{i}}^{\pi_{n}^{i}}(\bu)\|^{2}_{\pi_{n}^{i}} = \|\bu\|^{2}_{\pi_{0}^{i}}$ for any $\bu\in\cT_{\pi_{0}^{i}}$.
	\par
	In Wasserstein space, we define $\Gamma_{\mu}^{\pi}(\bu)= \bu\circ T_{\pi\to\mu}$ for $\mu, \pi\in\cP$ and $\bu\in\cT_{\mu}$, where $\circ$ is the composition operator, and $T_{\pi\to\mu}: \cX\to\cX$ satisfies $T_{\pi\to\mu}(\bx)\sim \mu$ for $\bx\sim \pi$. Thus, 
	\begin{equation}
		\small
		\begin{aligned}
			\|\bu\|_{\mu}^{2} = \int \|\bu\|^{2} d\mu = \int \|\bu(T_{\pi\to\mu}(\bx))\|^{2} d\pi(\bx) = \int \|\bu \circ T_{\pi\to \mu}\|^{2}d\pi = \int \|\Gamma_{\mu}^{\pi} (\bu)\|d\pi,
		\end{aligned}
	\end{equation}
	which is consistent with the definition of parallel transport. Based on the preceding notations, we now proceed to construct the continuous Riemannian SVRG flow derived from Algorithm \ref{alg:discrete riemannian svrg}. The underlying rationale aligns with the approaches established in Propositions \ref{pro:riemannian gradient} and \ref{pro:equivalence of SGD}, namely, approximating the discrete Riemannian SVRG updates with a deterministic flow in the Wasserstein space. The formal result is stated in the following proposition.
	\begin{restatable}{proposition}{svrgflow}\label{pro:riemannian svrg flow}
		Under Assumption \ref{ass:stochastic continuous}, let $f_{\xi}(\pi) = D_{KL}(\pi \parallel \mu_{\xi})$, the discrete Riemannian SVRG  Algorithm \ref{alg:discrete riemannian svrg} with $1 \leq n \leq \cO(\lfloor1 / \eta\rfloor)$ approximates the Riemannian SVRG flow 
		\begin{equation}\label{eq:svrg flow on manifold}
			\small
			\begin{aligned}
				\frac{\partial}{\partial{t}}\pi_{t}(\bx) = \nabla\cdot\left[\pi_{t}(\bx)\left(\nabla\log{\frac{d\pi_{t}}{d\mu}}(\bx)\right.\right. 
				\left.\left.- \frac{\eta}{2\pi_{t}(\bx)}\int \pi_{t_{i}, t}(\by, \bx)\nabla_{\bx}\cdot \Sigma_{\rm SVRG}(\by, \bx)d\by \right.\right. \\
				\left.\left. - \frac{\eta}{2\pi_{t}(\bx)}\int\pi_{t_{i}, t}(\by, \bx)\Sigma_{\rm SVRG}(\by, \bx)\nabla_{\bx}\log{\pi_{t_{i}, t}(\by, \bx)}d\by\right)\right],
			\end{aligned}
		\end{equation}
		for $iM\eta=t_{i} \leq t \leq t_{i + 1}$, $(\hbx_{t_{i}}, \hbx_{t})\sim \pi_{t_{i}, t}$, $\hbx_{t}\sim\pi_{t}$ in \eqref{eq:svrg flow on manifold}, since we have $\mE[\|\bx_{n}^{i} - \hbx_{(iM + n)\eta}\|^{2}] \leq \cO(\eta)$, where $\bx_{n}^{i}\sim \pi_{n}^i$ in Algorithm \ref{alg:discrete riemannian svrg} for $\bx_{0}^0 = \hbx_{0}$. Here    
		\begin{equation}\label{eq:svrg variance}	
			\begin{aligned}
				\Sigma_{\rm SVRG}(\by, \bx)
				& = \mE_{\xi}\left[\left(\nabla\log{\mu_{\xi}}(\bx) - \nabla\log{\mu_{\xi}}(\by) + \nabla\mE_{\xi}\left[\log{\mu_{\xi}}(\by)\right] - \nabla\mE_{\xi}\log{\mu_{\xi}}(\bx)\right)\right. \\
				& \left(\nabla\log{\mu_{\xi}}(\bx) - \nabla\log{\mu_{\xi}}(\by) + \nabla\mE_{\xi}\left[\log{\mu_{\xi}}(\by)\right] - \nabla\mE_{\xi}\log{\mu_{\xi}}(\bx)\right)^{\top}].
			\end{aligned}
		\end{equation}
	\end{restatable}
	As in \eqref{pro:equivalence of SGD}, the Riemannian SVRG flow defined in \eqref{pro:riemannian svrg flow} also constitutes a deterministic curve in the Wasserstein space, which can be explained through similar  discussion following Proposition \ref{pro:equivalence of SGD}. Consequently, all three continuous flows in the Wasserstein space, \eqref{eq:lagevin FP}, \eqref{eq:sgd flow on riemannian}, and \eqref{eq:svrg flow on manifold} describe deterministic trajectories, irrespective of any randomness in their discrete counterparts.
	\par
	Moreover, similar to the discussion after Proposition \ref{pro:equivalence of SGD}, although the discrete Riemannian SVRG method presents implementation challenges, it can be effectively approximated through both the discrete SVRG Langevin dynamics \citep{zou2018subsampled} and the proposed Riemannian SVRG flow in \eqref{eq:svrg flow on manifold}. \textbf{Consequently, our formulation presents a valuable analytical framework for both SVRG Langevin dynamics and discrete Riemannian SVRG algorithms}. Further details are in Appendix \ref{app:convergence of riemannian svrg flow}. 
	
	\subsection{Convergence of Riemannian SVRG Flow}
	In this subsection, we analyze the convergence rate of the Riemannian SVRG flow defined in equation \eqref{eq:svrg flow on manifold}. The main result, stated in Theorem \ref{thm:convergence of svrg on manifold} below, adopts the same notations as those introduced in Proposition \ref{pro:riemannian svrg flow}. 
	
	\begin{restatable}{theorem}{riemanniansvrgconvergence}\label{thm:convergence of svrg on manifold}
		Let $\pi_{t}$ follows Riemannian SVRG flow \eqref{eq:svrg flow on manifold}, $\mu$ defined in Proposition \eqref{eq:stochastic optima}, for sequences $\{t_{i}\}$ with $\Delta = t_{1} - t_{0} = \cdots = t_{I} - t_{I - 1} = \cO(1 / \sqrt{\eta})$, and $\eta T = I\Delta$ to run Riemannian SVRG flow for $I$ epochs. Then for any $T$, if the $\nabla\log{\mu_{\xi}}$ is Lipschitz continuous with coefficient $L_{2}$ for all $\xi$, under proper $\pi_{0}$, we have  
		\begin{equation}\label{eq:riemannian svrg convergence nonconvex}
			\small
			\begin{aligned}
				\frac{1}{\eta T}\sum\limits_{i = 1}^{I}\int_{t_{i}}^{t_{i + 1}}\left\|\grad D_{KL}(\pi_{t}\parallel \mu)\right\|^{2}_{\pi_{t}}dt \leq \frac{2D_{KL}(\pi_{0}\parallel \mu)}{\eta T}.
			\end{aligned}
		\end{equation}
		By taking $\eta = \cO(N^{-2/3})$, the computational complexity of Riemannian SVRG flow is of order $\cO(N^{2/3}/\epsilon)$ to make $\min_{0\leq t \leq \eta T} \left\|\grad D_{KL}(\pi_{t}\parallel \mu)\right\|^{2}_{\pi_{t}} \leq \epsilon$.
		\par
		Furthermore, when the log-Sobolev inequality \eqref{eq:log Sobolev inequality} is satisfied for $\mu$, we have 
		\begin{equation}\label{eq:riemannian svrg convergence pl}
			\small
			D_{KL}(\pi_{\eta T}\parallel \mu) \leq e^{-\gamma \eta T}D_{KL}(\pi_{0}\parallel \mu),
		\end{equation} 
		and it takes $\cO((N + \gamma^{-1}N^{2/3})\log{\epsilon^{-1}})$ computational complexity to make $D_{KL}(\pi_{\eta T}\parallel \mu)\leq \epsilon$. 
	\end{restatable}
	The proof of this theorem is deferred to Appendix \ref{app:convergence of riemannian svrg flow}. As shown, for non-convex problems, the computational complexity required to reach an $\epsilon$-stationary point of the Riemannian SVRG flow is $\cO(N^{2/3}\epsilon^{-1})$. This complexity can be lower than that of Riemannian GD flow, which is $\cO(N\epsilon^{-1})$, and Riemannian SGD flow, which is $\cO(\epsilon^{-2})$, when $\epsilon$ is sufficiently small (as clarified in Section \ref{sec:Constructing Riemannian SVRG Flow}).
	
	On the other hand, under the Riemannian PL-inequality (i.e., log-Sobolev inequality), the computational complexities required to achieve $D_{\text{KL}}(\pi_{\eta T}\parallel \mu) \leq \epsilon$ are $\cO(\gamma^{-1}N\log{\epsilon^{-1}})$ and $\cO(\gamma^{-1}\epsilon^{-1})$ for Riemannian GD and SGD flows. These can be improved by the Riemannian SVRG flow when $\gamma^{-1} \geq 1$ and $\cO(\epsilon\log{\epsilon^{-1}}) \leq \cO((\gamma N + N^{\frac{2}{3}})^{-1})$. 
	Besides that, it is worth noting that, no matter with/without the log-Sobolev inequality, the proved convergence rates match the results in Euclidean space \citep{orvieto2019continuous}, and computational complexities match the discrete SVRG in Euclidean space \citep{reddi2016stochastic}. 
	\par
	In Theorem \ref{thm:convergence of svrg on manifold}, the initial distribution $\pi_0$ is required to satisfy $\mE_{\pi_{t_{i}, t}}[\tr(\nabla^{2}\log{(d\pi_{t}/d\mu)}\Sigma_{\rm SVRG})] \leq \lambda_{t}\mE_{\pi_{t_{i}, t}}[\tr(\Sigma_{\rm SVRG})]$ where $\lambda_t$ is a polynomial function of $t$. Since no constraint is imposed on the order of $\lambda_t$, which may grow arbitrarily large, this assumption can be readily satisfied in practice. Further technical details are provided in Appendix \ref{app:convergence of riemannian svrg flow}.

	\par
	As in Section \ref{sec:convergence rate of riemannian sgd flow}, the convergence results in Theorem \ref{thm:convergence of svrg on manifold} do not require Lipschitz continuous $\log{\mu_{\xi}}$, so that Example \ref{example:gaussian} is suitable to be analyzed as follows. 
	\par
	As discussed in Section \ref{sec:Constructing Riemannian SVRG Flow}, the fundamental principle of SVRG lies in reducing the variance of stochastic gradient estimates at each update, thereby accelerating convergence. Interestingly, in Example \ref{example:gaussian}, the induced noise covariance $\Sigma_{\rm SVRG} = 0$, which causes the Riemannian SVRG flow \eqref{eq:svrg flow on manifold} degenerate into Riemannian GD flow \eqref{eq:lagevin FP}. This indicates that the variance reduction technique completely eliminates gradient variance in this case. Notably, since $\nabla \log{\mu_{\xi}}(\bx) = (\bx - \bxi)$ and $\nabla \log{\mu}(\bx) = (\bx - \bar{\bxi})$, the corresponding random vector $\bx_{n}^{i} \in \bbR^{d}$ in the discrete Riemannian SVRG (line 6 of Algorithm \ref{alg:discrete riemannian svrg}) satisfies. 
	\begin{equation*}
		\small
		\begin{aligned}
			\bx_{n + 1}^{i}	= \bx_{n}^{i} - \eta \left(\nabla\log{\frac{d\mu_{\xi_{n}^{i}}}{d\pi_{n}^{i}}}(\bx_{n}^{i}) - \nabla\log{\frac{d\mu_{\xi_{n}^{i}}}{d\pi_{0}^{i}}}(\bx_{0}^{i}) \!+\! \nabla\log{\frac{d\mu}{d\pi_{0}^{i}}}(\bx_{0}^{i})\right) = \bx_{n}^{i} + \eta\nabla\log{\frac{d\mu}{d\pi_{n}^{i}}(\bx_{n}^{i})}, 
		\end{aligned}
	\end{equation*}
	which is exactly the discretion of Riemannian GD flow, but with $\cO(1)$ (instead of $\cO(N)$ as we do not compute $\bar{\bxi}$ for each $n$) computational complexity for each update step in line 6 of Algorithm \ref{alg:discrete riemannian svrg}. This explains the improved computational complexity of Riemannian SVRG flow. Note that the corresponded SDE of \eqref{eq:svrg flow on manifold} is $	d\bx_{t} = -(\bx_{t} - \bar{\bxi})dt + \sqrt{2}dW_{t},$ 
	with closed-form solution $\bx_{\eta T} = \bar{\bxi} + e^{-\eta T}(\bx_{0} - \bar{\bxi}) + e^{-\eta T}W_{e^{2\eta T} - 1}$. Then, we can prove the convergence rates of $D_{KL}(\pi_{\eta T}\parallel \mu)$ ($\bx_{\eta T}\sim \pi_{\eta T}$) and its first order criteria are all of order $\cO(e^{-\eta T})$, where the global convergence rate matches the result in Theorem \ref{thm:convergence of svrg on manifold}. 
	\section{Experiments}
	\begin{figure*}[t!]\centering
		\begin{minipage}[t]{0.24\linewidth}
			\centering
			\includegraphics[width=1\linewidth]{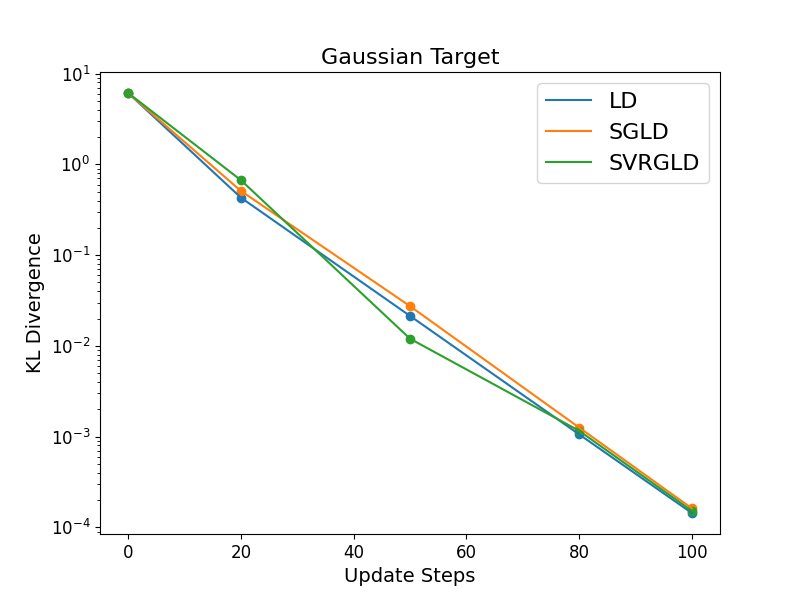}
			\vspace{0.02cm}
			\includegraphics[width=1\linewidth]{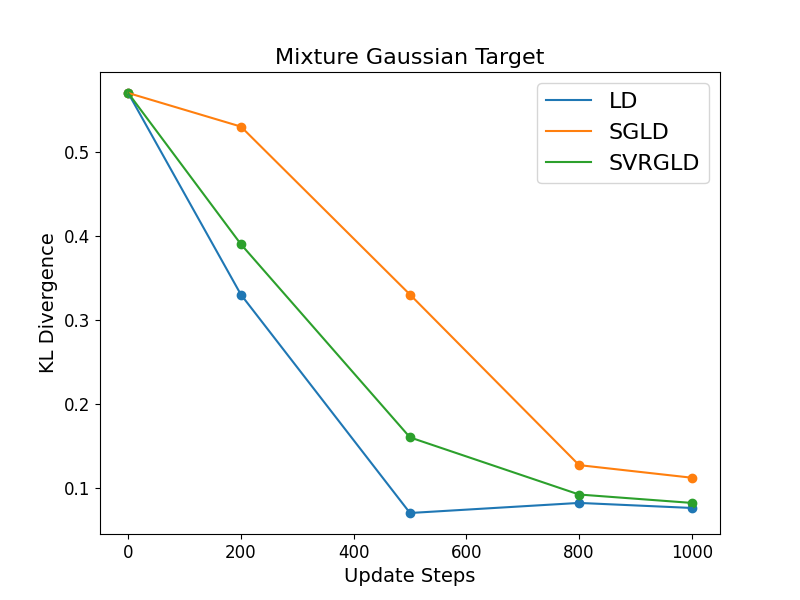}
			\vspace{0.02cm}
		\end{minipage}
		\begin{minipage}[t]{0.24\linewidth}
			\includegraphics[width=1\linewidth]{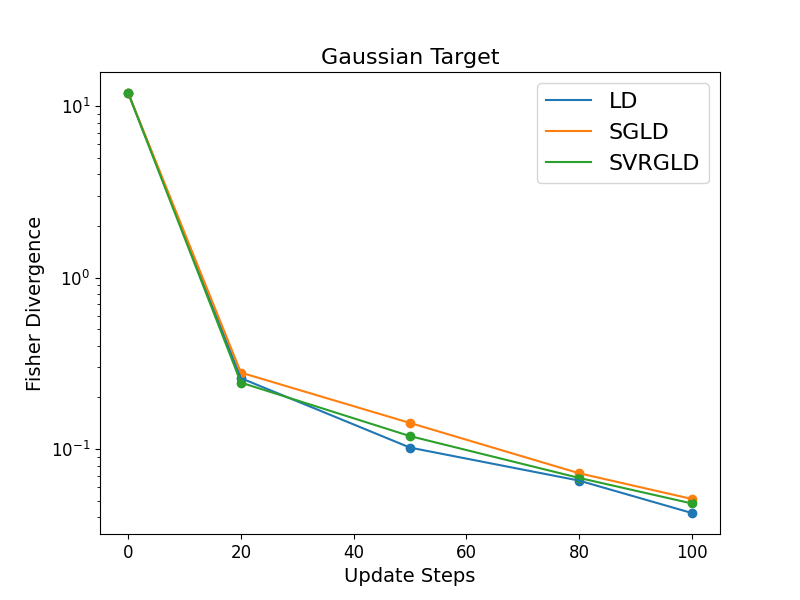}
			\vspace{0.02cm}
			\includegraphics[width=1\linewidth]{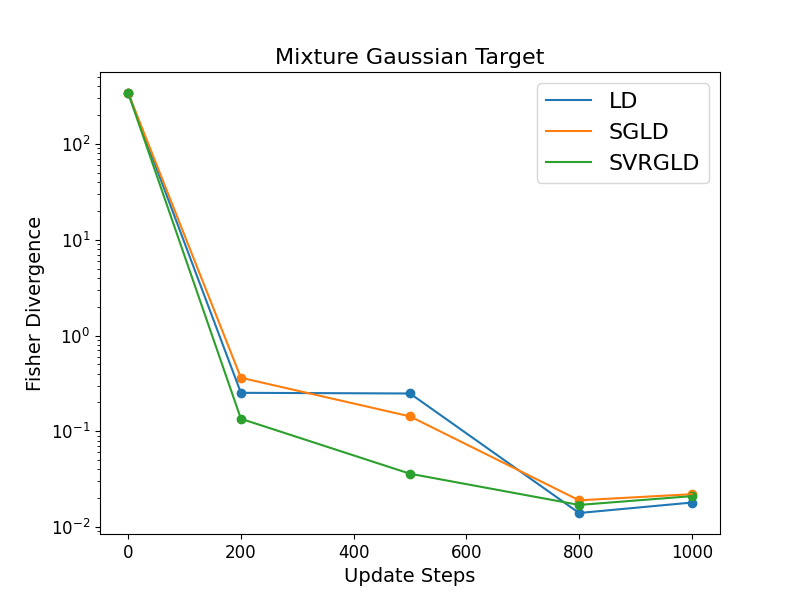}
			\vspace{0.02cm}
		\end{minipage}
		\begin{minipage}[t]{0.24\linewidth}
			\includegraphics[width=1\linewidth]{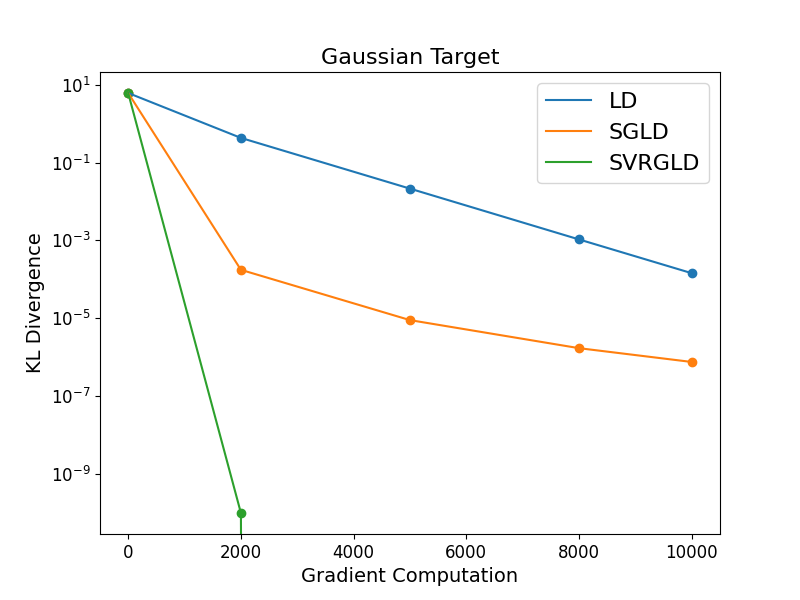}
			\vspace{0.02cm}
			\includegraphics[width=1\linewidth]{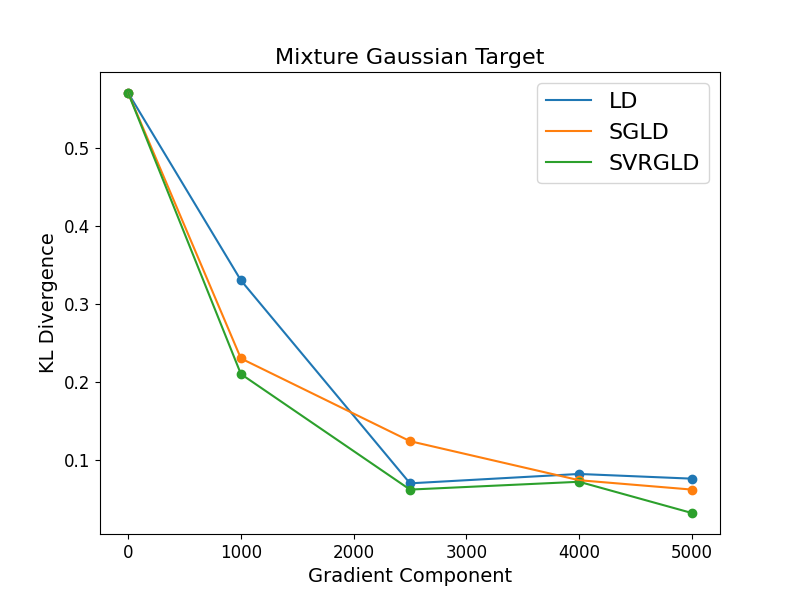}
			\vspace{0.02cm}
		\end{minipage}
		\begin{minipage}[t]{0.24\linewidth}
			\includegraphics[width=1\linewidth]{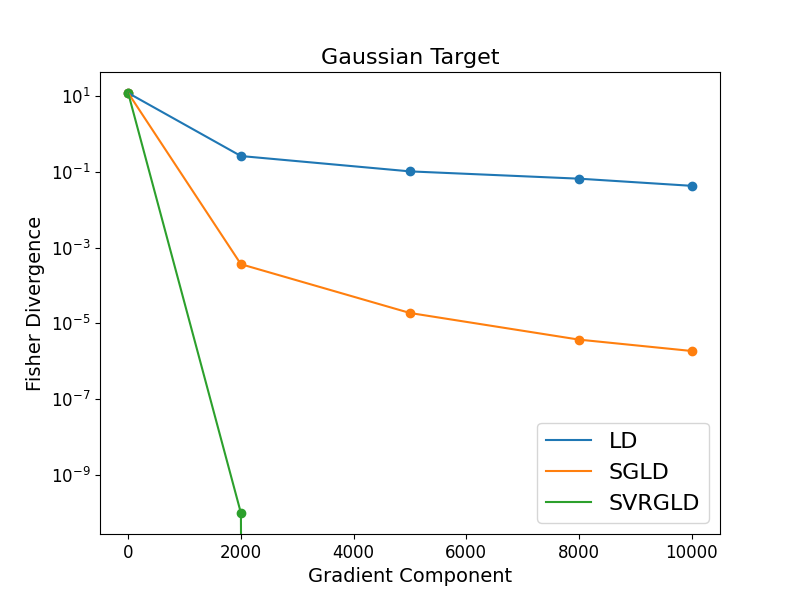}
			\vspace{0.02cm}
			\includegraphics[width=1\linewidth]{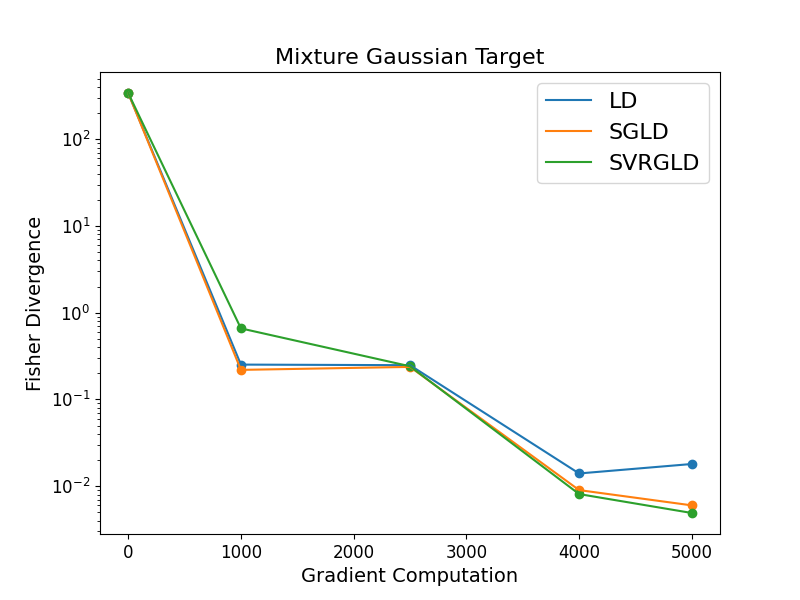}
			\vspace{0.02cm}
		\end{minipage}
		\caption{The convergence rates measured by KL divergence and Fisher divergence (Riemannian gradient norm) under different optimization methods.}
		\label{fig:convergence}
	\end{figure*}

	In this section, we empirically evaluate the proposed sampling algorithms on two examples. 
	\paragraph{Gaussian.} In this case, we set $N = 100$, $d=2$,  $\mu_{\bxi_{i}}\sim (\bxi_{i}, \bI)$  with each $\bxi_{i}\sim \cN(\boldmath{0}, \bI)$. Then, we get the target distribution $\mu\sim \cN(\bar{\bxi}, \bI)$. Within this framework, we report the KL divergence $D_{KL}(\pi_{\eta T}\parallel \mu)$ and the Fisher divergence  $\|\grad D_{KL}(\pi_{\eta T}\parallel \mu)\|^{2}$ (first order criteria), where the $\pi_{\eta T}$ is obtained by Riemannian GD flow, Riemannian SGD flow or SVRG flow. In this case, due to Example \ref{example:gaussian}, the two criteria can be explicitly estimated. The results are summarized in Figure \ref{fig:convergence}.
	
	\paragraph{Mixture Gaussian.} In this case, we set $N = 5$, $d = 2$ with $\mu_{\bxi_{i}}\sim \frac{1}{2}\cN(\bxi_{i,1}, \bI) + \frac{1}{2}\cN(\bxi_{i,2}, \bI)$, where $\bxi_{i,k}\sim \cN(0, \bI)$. Unfortunately, the proposed flows can not be explicitly computed. Therefore, we implement their discrete versions as in Algorithms \ref{alg:discrete riemannian sgd}, \ref{alg:discrete riemannian svrg}. Then we report the KL divergence $D_{KL}(\pi_{\eta T}\parallel \mu)$ and Fisher divergence $\|\grad D_{KL}(\pi_{\eta T}\parallel \mu)\|^{2}$ with $\pi_{\eta T}$ are approximated by the discrete Riemannian GD, SGD, and SVRG Algorithms. Here, the KL divergence and Riemannian gradient are estimated by density estimation as in \citep{}, with $1000$ independent samples. The results are summarized in Figure \ref{fig:convergence}. 
	
	As can be seen, under the same update steps, Riemannian SVRG and Riemannian GD have better convergence results than Riemannian SGD (Theorem \ref{thm:convergence riemannian gd}, \ref{thm:convergence riemannian sgd}, and \ref{thm:convergence of svrg on manifold}), while under the same gradient computations, Riemannian SGD and Riemannian SVRG (especially Riemannian SVRG) are sharper (see discussion after Theorem \ref{thm:convergence of svrg on manifold}). These results are consistent with theoretical conclusions.

	\section{Conclusion}
	Based on the principles of Riemannian manifold optimization, this paper investigates continuous Riemannian SGD and SVRG flows for minimizing KL divergence within the Wasserstein space. We establish convergence rates for these stochastic flows that align with known results in Euclidean settings. Our technical approach involves constructing SDEs in Euclidean space by taking the limit of vanishing step sizes in discrete Riemannian optimization methods, where the corresponding Fokker-Planck equations characterize the desired curves in Wasserstein space. This framework provides new theoretical insights into continuous stochastic Riemannian optimization and demonstrates the utility of continuous methods as analytical tools for studying discrete optimization algorithms. 
	\newpage
	
	\bibliographystyle{plain}
	\bibliography{reference}
	
	\section*{NeurIPS Paper Checklist}
	\begin{enumerate}
		
		\item {\bf Claims}
		\item[] Question: Do the main claims made in the abstract and introduction accurately reflect the paper's contributions and scope?
		\item[] Answer:  
		\item[] Justification: The main contributions of this paper have been clarified in Abstract. 
		
		\item {\bf Limitations}
		\item[] Question: Does the paper discuss the limitations of the work performed by the authors?
		\item[] Answer: \answerYes{} 
		\item[] Justification: The limitation of this paper is discussed in the main paper.
		
		\item {\bf Theory Assumptions and Proofs}
		\item[] Question: For each theoretical result, does the paper provide the full set of assumptions and a complete (and correct) proof?
		\item[] Answer: \answerYes{} 
		\item[] Justification: All results are proofed. 
		
		\item {\bf Experimental Result Reproducibility}
		\item[] Question: Does the paper fully disclose all the information needed to reproduce the main experimental results of the paper to the extent that it affects the main claims and/or conclusions of the paper (regardless of whether the code and data are provided or not)?
		\item[] Answer: \answerNA{} 
		\item[] Justification: 

		\item {\bf Open access to data and code}
		\item[] Question: Does the paper provide open access to the data and code, with sufficient instructions to faithfully reproduce the main experimental results, as described in supplemental material?
		\item[] Answer: \answerNA{} 
		\item[] Justification: 

		\item {\bf Experimental Setting/Details}
		\item[] Question: Does the paper specify all the training and test details (e.g., data splits, hyperparameters, how they were chosen, type of optimizer, etc.) necessary to understand the results?
		\item[] Answer: \answerNA{} 
		\item[] Justification: 
		
		\item {\bf Experiment Statistical Significance}
		\item[] Question: Does the paper report error bars suitably and correctly defined or other appropriate information about the statistical significance of the experiments?
		\item[] Answer: \answerNA{} 
		\item[] Justification:  
		
		\item {\bf Experiments Compute Resources}
		\item[] Question: For each experiment, does the paper provide sufficient information on the computer resources (type of compute workers, memory, time of execution) needed to reproduce the experiments?
		\item[] Answer: \answerNA{} 
		\item[] Justification: 
		
		\item {\bf Code Of Ethics}
		\item[] Question: Does the research conducted in the paper conform, in every respect, with the NeurIPS Code of Ethics \url{https://neurips.cc/public/EthicsGuidelines}?
		\item[] Answer: \answerYes{} 
		\item[] Justification:

		\item {\bf Broader Impacts}
		\item[] Question: Does the paper discuss both potential positive societal impacts and negative societal impacts of the work performed?
		\item[] Answer: \answerNA{} 
		\item[] Justification:
		
		\item {\bf Safeguards}
		\item[] Question: Does the paper describe safeguards that have been put in place for responsible release of data or models that have a high risk for misuse (e.g., pretrained language models, image generators, or scraped datasets)?
		\item[] Answer: \answerNA{} 
		\item[] Justification: 
		
		\item {\bf Licenses for existing assets}
		\item[] Question: Are the creators or original owners of assets (e.g., code, data, models), used in the paper, properly credited and are the license and terms of use explicitly mentioned and properly respected?
		\item[] Answer: \answerNA{} 
		\item[] Justification: 
		
		\item {\bf New Assets}
		\item[] Question: Are new assets introduced in the paper well documented and is the documentation provided alongside the assets?
		\item[] Answer: \answerNA{} 
		\item[] Justification: 
		
		\item {\bf Crowdsourcing and Research with Human Subjects}
		\item[] Question: For crowdsourcing experiments and research with human subjects, does the paper include the full text of instructions given to participants and screenshots, if applicable, as well as details about compensation (if any)? 
		\item[] Answer: \answerNA{} 
		\item[] Justification:
		
		\item {\bf Institutional Review Board (IRB) Approvals or Equivalent for Research with Human Subjects}
		\item[] Question: Does the paper describe potential risks incurred by study participants, whether such risks were disclosed to the subjects, and whether Institutional Review Board (IRB) approvals (or an equivalent approval/review based on the requirements of your country or institution) were obtained?
		\item[] Answer: \answerNA{} 
		\item[] Justification: 
		
	\end{enumerate}
	\newpage
	
	\appendix
	\include{appendix}
	
\end{document}

%% file: appendix.tex
\onecolumn
\section{Proofs in Section \ref{sec:riemannian gradient flow}}\label{app:proofs in riemannian gradient flow}
\riemanniangradient*
\begin{proof}
	For any curve $\pi_{t}\in\cP$ with $\pi_{0} = \pi$, we have,  
	\begin{equation}\label{eq:riemannian gradient inner product}
		\small
		\begin{aligned}
			\lim_{t\to 0} \frac{D_{KL}(\pi_{t}\parallel \mu) - D_{KL}(\pi_{0}\parallel \mu)}{t} & = \mathrm{Dev}D_{KL}(\pi_{0} \parallel \mu)[\bv_{0}] \\
			& = \frac{\partial}{\partial{t}}D_{KL}(\pi_{0} \parallel \mu)\mid_{t = 0} \\
			& = \int (1 + \log{\pi_{t}})\frac{\partial}{\partial{t}}\pi_{t}d\bx\mid_{t = 0} + \int \frac{\partial}{\partial{t}}\pi_{t}\log{\mu}d\bx\mid_{t = 0}\\
			& = \int \langle\nabla\log{\pi} - \nabla\log{\mu}, \bv_{0}\rangle d\pi \\
			& = \left\langle\bv_{0}, \nabla\log{\frac{d\pi}{d\mu}}\right\rangle_{\pi}.
		\end{aligned}
	\end{equation}
	Thus we prove our conclusion due to the definition of Riemannian gradient. 
\end{proof}

\riemanniangradientflow*
\begin{proof}
	For any $f$, due to the definition of $\pi_{n + 1}$ in \eqref{eq:exp map in manifold}, we have 
	\begin{equation}
		\small
		\mE_{\pi_{n + 1}}[f(\bx)] = \mE_{\pi_{n}}[f(\bx - \eta\grad F(\pi_{n}(\bx)))],
	\end{equation}
	so that for $\bx_{n}\sim \pi_{n}$ and $\bx_{n + 1}\sim \pi_{n + 1}$, we must have 
	\begin{equation}\label{eq:discrete ode x}
		\small
		\bx_{n + 1} = \bx_{n} - \eta\grad F(\pi_{n}(\bx_{n})) = \bx_{n} - \eta \nabla\log{\frac{d\pi_{n}}{d\mu}}(\bx_{n}). 
	\end{equation}
	On the other hand, let us define 
	\begin{equation}
		\small
		\bar{\bx}_{n + 1} = \bar{\bx}_{n} + \eta\nabla_{\bx}\log{\mu(\bar{\bx}_{n})} + \sqrt{2\eta}\beps_{n},
	\end{equation}
	where $\beps_{n}\sim \cN(0, \bI)$, and $\bar{\bx}_{0} = \bx_{0}$. Next, let us show $\bx_{n}$ approximates $\bar{\bx}_{n}$. For any test function $f\in C^{2}$ with (spectral norm) bounded Hessian, we have 
	\begin{equation}\label{eq:hax x_n approximation}
		\small
		\begin{aligned}
			\mE\left[f(\bx_{n + 1})\right] & = \mE[f(\bx_{n})] + \eta\mE\left[\left\langle\nabla f(\bx_{n}), \nabla_{\bx}\log{\mu}(\bx_{n}) - \nabla_{\bx}\log{\pi_{n}(\bx_{n})}\right\rangle\right] + \cO(\eta^{2}) \\
			& = \mE[f(\bx_{n})] + \eta\mE\left[\left\langle\nabla f(\bx_{n}), \nabla_{\bx}\log{\mu}(\bx_{n})\right\rangle\right] + \eta\mE[\Delta f(\bx_{n})] + \cO(\eta^{2}) \\
			& = \mE[f(\bx_{n})] + \eta\mE\left[\left\langle\nabla f(\bx_{n}), \nabla_{\bx}\log{\mu}(\bx_{n})\right\rangle\right] + \eta\mE\left[\beps_{n}^{\top}\nabla^{2}f(\bx_{n})\beps_{n}\right] + \cO(\eta^{2}) \\
			& = \mE\left[f\left(\bx_{n} + \eta\nabla_{\bx}\log{\mu(\bx_{n})} + \sqrt{2\eta}\beps_{n}\right)\right] + \cO(\eta^{2}). 
		\end{aligned}
	\end{equation}
	due to the definition of $f$ and $D_{KL}(\pi_{n}\parallel \mu) < \infty$. Then, let us define the following 
	\begin{equation}\label{eq:f_delta_C}
		\small
			f_{\delta, C}(\bx) = \begin{dcases}
				\|\bx\|^{2} &\qquad \|\bx\|^{2} \leq C - \delta; \\
				u_{\delta}(\bx) &\qquad C - \delta \leq \|\bx\|^{2} \leq C; \\
				C & \qquad C > \|\bx\|^{2}, 
			\end{dcases}
	\end{equation}
	where $\delta > 0$, $C > 0$ and $u_{\delta}(\bx)$ is a quadratic function of $\bx$ to make the above $f_{\delta, C}$ smooth. Then 
	\begin{equation}\label{eq:recur gd riemannian}
		\small
		\begin{aligned}
			\mE&\left[f_{\delta, C}(\bx_{n + 1} - \barx_{n + 1})\right]  = \mE\left[f_{\delta, C}\left(\bx_{n} + \eta\nabla_{\bx}\log{\mu(\bx_{n})} + \sqrt{2\eta}\beps_{n} - \barx_{n + 1}\right)\right] + \cO(\eta^{2})\\
			& =  \mE\left[f_{\delta, C}\left(\bx_{n} + \eta\nabla_{\bx}\log{\mu(\bx_{n})} - \barx_{n} - \eta\nabla_{\bx}\log{\mu(\barx_{n})}\right)\right] + \cO(\eta^{2}) \\
			& \leq (1 + \eta)\mE\left[f_{\delta, C}\left(\bx_{n} - \bar{\bx}_{n}\right)\right] + \left(1 + \frac{1}{\eta}\right)\eta^{2}\mE\left[f_{\delta, C}\left(\nabla_{\bx}\log{\mu(\bx_{n})} - \nabla_{\bx}\log{\mu(\bar{\bx}_{n})}\right)\right] + \cO(\eta^{2}) \\
			& \leq (1 + \cO(\eta)) \mE\left[f_{\delta, C}\left(\bx_{n} - \bar{\bx}_{n}\right)\right] + \cO(\eta^{2}),
		\end{aligned}
	\end{equation}
	where the last two inequalities are respectively from Young's inequality $\|\ba + \bb\|^{2} \leq (1 + \eta)\|\ba\|^{2} + (1 + 1 / \eta)\|\bb\|^{2}$ for any $\eta > 0$, the Lipschitz continuity of $\nabla_{\bx}\log{\mu}$, and the definition of $f_{\delta, C}$. Then, by recursively using the above inequality, we have 
	\begin{equation}
		\small
		\mE[f_{\delta, C}(\bx_{n} - \barx_{n})] \leq \cO(\eta)
	\end{equation}
	when $n \leq \cO(1 / \eta)$. By taking $\delta \to 0$, $C\to \infty$, and applying Fatou's Lemma \citep{shiryaev2016probability}, we get 
	\begin{equation}\label{eq:Fatou lemma}
		\small
			\mE\left[\left\|\bx_{n} - \barx_{n}\right\|^{2}\right] = \mE\left[\varlimsup_{\delta\to0, C\to\infty}f_{\delta, C}(\bx_{n} - \barx_{n})\right] \leq \varlimsup_{\delta\to0, C\to\infty}\mE\left[f_{\delta, C}(\bx_{n} - \barx_{n})\right] \leq \cO(\eta). 
	\end{equation}
	\par
	Next, we should show the $\bar{\bx}_{n}$ (so that $\bx_{n}$) is a discretion of the following SDE (continuous Langevin dynamics). Let us define 
	\begin{equation}\label{eq:gradient flow ode}
		\small
		d\hat{\bx}_{t} = \nabla\log{\mu}(\hat{\bx}_{t})dt + \sqrt{2}dW_{t}. 
	\end{equation}
	For any given $\barx_{0} = \hat{\bx}_{0}$, similar to \eqref{eq:sgd discrete error}, we can prove 
	\begin{equation}
		\small
		\begin{aligned}
			\mE&\left[\left\|\hat{\bx}_{(n + 1)\eta} - \barx_{n + 1}\right\|^{2}\right]  \leq \mE\left[\left\|\hat{\bx}_{n\eta} + \eta\nabla\log{\mu(\hat{\bx}_{n\eta})} - \barx_{n + 1}\right\|^{2}\right] + \cO(\eta^{2}) \\
			& \leq (1 + \eta)\mE\left[\left\|\hat{\bx}_{n\eta} - \barx_{n}\right\|^{2}\right] + \left(1 + \frac{1}{\eta}\right)\mE\left[\left\|\nabla\log{\mu(\hat{\bx}_{n\eta})} - \nabla\log{\mu(\barx_{n})}\right\|^{2}\right] + \cO(\eta^{2}) \\
			& \leq (1 + \eta)\mE\left[\left\|\hat{\bx}_{n\eta} - \barx_{n}\right\|^{2}\right] + L_{2}^{2}\eta^{2}\left(1 + \frac{1}{\eta}\right)\mE\left[\left\|\hat{\bx}_{n\eta}- \barx_{n}\right\|^{2}\right] + \cO(\eta^{2}) \\
			& = (1 + \cO(\eta))\mE\left[\left\|\hat{\bx}_{n\eta}- \barx_{n}\right\|^{2}\right] + \cO(\eta^{2}). 
		\end{aligned}
	\end{equation}
	By iteratively applying this inequality, we get $\mE[\|\hat{\bx}_{n\eta} - \bx_{n}\|^{2}] \leq \cO(\eta)$ for any $n\le \cO(1 / \eta)$. Thus, by the  triangle inequality
	\begin{equation}
		\small
		\mE\left[\|\hbx_{n\eta} - \bx_{n}\|^{2}\right] \leq 2\mE\left[\|\hbx_{n\eta} - \barx_{n}\|^{2}\right] + 2\mE\left[\|\barx_{n} - \bx_{n}\|^{2}\right] \leq \cO(\eta). 
	\end{equation}
	Thus, we prove our conclusion by applying Lemma \ref{lem:equivalence} to \eqref{eq:gradient flow ode}. 
\end{proof}
\begin{remark}
	It is worthy to note that during our proof, we introduce the auxiliary sequence $\{\barx_{n}\}$ which is the discrete Langevin dynamics. We construct its continuous counterpart instead of $\bx_{n}$, because the drift term of it is $\nabla_{\bx}\log{\mu}$ which has verifiable continuity. However, if we directly analyze $\bx_{n}$ with drift term $\nabla_{\bx}\log{\mu/\pi_{n}}$, its continuity is non-verifiable. 
\end{remark}

\convergencergd*
\begin{proof}
	Similar to \eqref{eq:riemannian gradient inner product}, we have 
	\begin{equation}
		\small
		\frac{\partial}{\partial{t}}D_{KL}(\pi_{t} \parallel \mu) = \mathrm{Dev}D_{KL}(\pi_{t} \parallel \mu)[-\grad D_{KL}(\pi_{t} \parallel \mu)] = -\left\|\grad D_{KL}(\pi_{t} \parallel \mu)\right\|^{2}_{\pi_{t}}.
	\end{equation}
	Thus we know the Riemannian gradient flow resulted $\pi_{t}$ is monotonically decreased w.r.t. $t$ 
	Taking integral and from the non-negativity of KL divergence implies the conclusion. 
	\par
	On the other hand, if the Riemannian PL inequality \eqref{eq:riemannian pl} holds with coefficient $\gamma$, then the above equality further implies 
	\begin{equation}
		\small
		\frac{\partial}{\partial{t}}D_{KL}(\pi_{t} \parallel \mu) \leq -2\gamma D_{KL}(\pi_{t} \parallel \mu). 
	\end{equation}
	So that taking integral implies the second conclusion. 
\end{proof}

\section{Proofs in Section \ref{sec:riemannian sgd flow}}\label{app:proofs in section riemannian sgd flow}
\stochasticoptima*
\begin{proof}
	The result is directly proved by Langrange's multiplier theorem. Due to the definition of KL divergence and \eqref{eq:stochastic objective}, the optimal $\pi$ should satisfies 
	\begin{equation}
		\small
			\frac{\partial}{\partial{\pi}}\cL(\pi) = \frac{\partial}{\partial{\pi}}\left\{\mE_{\xi}\left[D_{KL}(\pi \parallel \mu_{\xi})\right] - \lambda\left(\int \pi(\bx)d\bx - 1\right)\right\} = 0,
	\end{equation}
	which results in 
	\begin{equation}
		\small
		\int p(\xi)\int (1 + \log\pi(\bx)) - \log{\mu_{\xi}(\bx)} - \lambda d\bx d\xi = 0
	\end{equation}
	for some $\lambda > 0$, where $p(\xi)$ is the density of $\xi$. Change the order of integral, we know that for any $\bx$,  
	\begin{equation}
		\small
		\log{\pi(\bx)} + (1 - \lambda) = \int p(\xi)\log{\mu_{\xi}}(\bx)d\xi, 
	\end{equation}
	which indicates our conclusion under the condition of $\int \pi(\bx) d\bx = 1$. 
\end{proof}

\subsection{Proofs of Proposition \ref{pro:equivalence of SGD}}
\riemanniansgdflow*
\begin{proof}
	Similar to the proof of Proposition \ref{pro:riemannian gradient}, we can show that for  
	\begin{equation}\label{eq:barx discrete sgd}
		\small
		\barx_{n + 1} = \barx_{n} + \eta \nabla\log\mu_{\xi_{n}}(\barx_{n}) + \sqrt{2\eta}\beps_{n}, 
	\end{equation}
	with $\beps_{n}\sim\cN(0, \bI)$, we have $\bx_{n}$ approximates $\barx_{n}$. That says, similar to \eqref{eq:hax x_n approximation}, when $\bx_{0} = \bar{\bx}_{0}$, we can prove
	\begin{equation}\label{eq:approxmiation barx sgd}
		\small
		\mE\left[\|\bx_{n} - \barx_{n}\|^{2}\right] \leq \cO(\eta^{2})
	\end{equation}
	Next, we show $\barx_{n}$ in \eqref{eq:barx discrete sgd} is the discretion of stochastic differential equation 
	\begin{equation}\label{eq:continuous sde for sgd}
		\small
		\begin{aligned}
			d\hbx_{t} &= \nabla\mE_{\xi}\left[\log\mu_{\xi}(\hbx_{t})\right] + \left(\sqrt{\eta}\Sigma^{\frac{1}{2}}_{\rm SGD}(\hbx_{t}), \sqrt{2}\bI\right)dW_{t} \\
			& = \bb(\hbx_{t}) + \left(\sqrt{\eta}\Sigma^{\frac{1}{2}}_{\rm SGD}(\hbx_{t}), \sqrt{2}\bI\right)dW_{t},
		\end{aligned}
	\end{equation}
	where $\Sigma_{\rm SGD}(\hbx_{t})$ is the covariance matrix 
	\begin{equation}
		\small
		\begin{aligned}
			\Sigma_{\rm SGD}(\hbx_{t}) 
			& = \mE_{\xi}\left[\left(\nabla\log{\mu_{\xi}}(\hbx_{t}) - \nabla\mE_{\xi}\left[\log{\mu_{\xi}}(\hbx_{t})\right]\right)\left(\nabla\log{\mu_{\xi}}(\hbx_{t}) - \nabla\mE_{\xi}\left[\log{\mu_{\xi}}(\hbx_{t})\right]\right)^{\top}\right].
		\end{aligned}
	\end{equation}
	To check this, for any test function $f\in C^{2}$ with bounded gradient and Hessian, due to $\barx_{0} = \hat{\bx}_{0}$, Dynkin's formula \citep{oksendal2013stochastic}, we have 
	\begin{equation}\label{eq:continuous expansion on sde}
		\small
		\begin{aligned}
			\mE\left[f(\hbx_{(n + 1)\eta})\right] & = \mE\left[f(\hbx_{n\eta})\right] + \int_{n\eta}^{(n + 1)\eta}\mE\left[\cL f(\hbx_{t})\right]dt\\
			& = \mE\left[f(\hbx_{n\eta})\right] + \eta\mE\left[\cL f(\hbx_{n\eta})\right] + \frac{1}{2}\int_{n\eta}^{(n + 1)\eta}\int_{n\eta}^{t}\mE\left[\cL^{2}f(\hbx_{s})\right]dsdt \\
			& = \mE\left[f(\hbx_{n\eta})\right] + \eta\mE\left[\cL f(\hbx_{n\eta})\right] + \cO(\eta^{2}), 
		\end{aligned}
	\end{equation}
	where the last equality is due to the Lipschitz continuity of $\nabla \log{\mu_{\xi}}(\bx)$ indicates $\nabla^{2}\log{\mu_{\xi}(\bx)}$ have an upper bounded spectral norm, and $\nabla \log{\mu_{\xi}}(\bx)$ itself is upper bounded. Noting that 
	\begin{equation}
		\small
		\mE[\cL f(\hbx_{n\eta})] = \mE_{\hbx_{n\eta}}\left[\left\langle \bb(\hbx_{n\eta}), \nabla f(\hbx_{n\eta})\right\rangle\right] + \frac{\eta}{2}\mE\left[\tr\left(\Sigma_{\rm SGD}(\hbx_{n\eta})\nabla^{2}f(\hbx_{n\eta})\right)\right] + \mE\left[\Delta f(\hbx_{n\eta})\right].
	\end{equation}
	Plugging these into \eqref{eq:continuous expansion on sde}, we get 
	\begin{equation}
		\small
		\begin{aligned}
			\mE\left[f(\hbx_{(n + 1)\eta})\right] & = \mE\left[f(\hbx_{n\eta})\right] + \mE_{\hbx_{n\eta}}\left[\left\langle \bb(\hbx_{n\eta}), \nabla f(\hbx_{n\eta})\right\rangle\right] + \frac{\eta}{2}\mE\left[\tr\left(\Sigma_{\rm SGD}(\hbx_{n\eta})\nabla^{2}f(\hbx_{n\eta})\right)\right]\\
			& + \mE\left[\Delta f(\hbx_{n\eta})\right] + \cO(\eta^{2}).
		\end{aligned}
	\end{equation}
	On the other hand, we can similarly prove that 
	\begin{equation}
		\small
		\begin{aligned}
			\mE& \left[f(\hbx_{n\eta} + \eta\nabla\log\mu_{\xi_{n}}(\hbx_{n\eta}) + \sqrt{2\eta}\beps_{n})\right] = \mE\left[f(\hbx_{n\eta})\right] + \eta\mE[\langle \bb(\hbx_{n\eta}), \nabla f(\hbx_{n\eta})\rangle] + \eta\mE[\Delta f(\hbx_{n\eta})]\\
			& + \frac{\eta^{2}}{2}\mE\left[\tr\left[\left(\Sigma_{\rm SGD}(\hbx_{n\eta}) + \bb(\hbx_{n\eta})\bb^{\top}(\hbx_{n\eta})\right)\nabla^{2}f(\hbx_{n\eta})\right]\right] + \cO(\eta^{3}). 
		\end{aligned}
	\end{equation}
	Thus we get 
	\begin{equation}\label{eq:sgd discrete error}
		\small
		\sup_{\bx}\left|\mE^{\bx}\left[f(\hbx_{(n + 1)\eta})\right] - \mE^{\bx}\left[f(\hbx_{n\eta} + \eta\nabla\log\mu_{\xi_{n}}(\hbx_{n\eta}) + \sqrt{2\eta}\beps_{n})\right]\right| = \cO(\eta^{2}), 
	\end{equation}
	due to the Lipschitz continuity of $\log{\mu_{\xi}}$, where $\mE^{\bx}[f(\hbx_{t})] = \mE[f(\hbx_{t})\mid \hbx_{0} = \bx]$. Let $f_{\delta, C}(\bx)$ be the ones in \eqref{eq:f_delta_C}, then similar to \eqref{eq:recur gd riemannian}, 
	\begin{equation}\label{eq:riemannian sgd discrete gap}
		\small
		\begin{aligned}
			\mE& \left[f_{\delta, C}\left(\hbx_{(n + 1)\eta} - \barx_{n + 1}\right)\right] 
			\leq \mE\left[f_{\delta, C}\left(\hbx_{n\eta} + \eta\nabla\log\mu_{\xi_{n}}(\hbx_{n\eta}) + \sqrt{2\eta}\beps_{n} - \barx_{n + 1}\right)\right] + \cO(\eta^{2}) \\
			& = \mE\left[f_{\delta, C}\left(\hbx_{n\eta} - \barx_{n} + \eta\nabla\log\mu_{\xi_{n}}(\hbx_{n\eta}) - \eta\nabla\log\mu_{\xi_{n}}(\barx_{n})\right)\right] + \cO(\eta^{2}) \\
			& \leq (1 + \cO(\eta))\mE\left[f_{\delta, C}\left(\hbx_{n\eta} - \barx_{n}\right)\right] + \left(1 + \frac{1}{\eta}\right)\eta^{2}\mE\left[f_{\delta, C}\left(\nabla_{\bx}\log{\mu_{\xi_{n}}(\hbx_{n\eta})} - \nabla_{\bx}\log{\mu_{\xi_{n}}(\bar{\bx}_{n})}\right)\right] + \cO(\eta^{2})\\
			& \leq (1 + \cO(\eta))\mE\left[f_{\delta, C}\left(\hbx_{n\eta} - \barx_{n}\right)\right] + \cO(\eta^{2}) \\
			& \leq \cdots \\
			& \leq \cO(\eta).
		\end{aligned}
	\end{equation}
	The above inequality holds for any $n\leq 1 / \eta$. Similar to \eqref{eq:Fatou lemma}, by taking $\delta\to 0, C\to \infty$, and applying Fatou's Lemma, we get 
	\begin{equation}
		\small
			\mE\left[\|\hbx_{n\eta} - \barx_{n}\|^{2}\right] \leq \cO(\eta). 
	\end{equation}
	Then combining \eqref{eq:approxmiation barx sgd} with the above inequality, and by triangle inequality, we have 
	\begin{equation}
		\small
		\mE\left[\|\bx_{n} - \hbx_{n\eta}\|^{2}\right] \leq \cO(\eta). 
	\end{equation}
	Thus we prove our conclusion by combining Lemma \ref{lem:equivalence}. 
\end{proof}
\subsection{Convergence of (Riemannian) SGD Flow}\label{app:convergence sgd}
In this subsection, we first give proof of the convergence rate of the SGD flow in Euclidean space, by transferring it into its corresponding stochastic ordinary equation. 
\par
Similar to Proposition \ref{pro:equivalence of SGD}, we can prove, in Euclidean space, the SGD flow of minimizing $F(\bx) = \mE_{\xi}[f_{\xi}(\bx)]$ takes the form of  
\begin{equation}
	\small
	d\bx_{t} = -\nabla F(\bx_{t}) + \sqrt{\eta}\Sigma_{\rm SGD}(\bx_{t})^{\frac{1}{2}}dW_{t},
\end{equation}
with $\Sigma_{\rm SGD}(\bx_{t}) = \mE_{\xi}\left[(\nabla f_{\xi}(\bx_{t}) - \mE_{\xi}[f_{\xi}(\bx)])(\nabla f_{\xi}(\bx_{t}) - \mE_{\xi}[f_{\xi}(\bx)])^{\top}\right]$. Then by Lemma \ref{lem:equivalence}, it's corresponded stochastic ordinary equation is 
\begin{equation}\label{eq:SGD ODE flow Euclidean space}
	\small
	d\bx_{t} = -\nabla F(\bx_{t}) - \frac{\eta}{2}\nabla\cdot\Sigma_{\rm SGD}(\bx_{t}) - \frac{\eta}{2}\Sigma_{\rm SGD}\nabla\log{\pi_{t}(\bx_{t})}dt.
\end{equation}
Before providing our theorem, we clarify the definition of our computational complexity to continuous optimization methods. Owing to the connection between the continuous method and its discrete counterpart as in Proposition \ref{pro:riemannian gradient}, \ref{pro:equivalence of SGD}, and \ref{pro:riemannian svrg flow}. It requires $T$ discrete update steps to arrive $\hbx_{\eta T}$. Therefore, the computational complexity of running $T$ discrete steps is said to be the computational complexity of continuous optimization methods measured under $\hbx_{\eta T}$.
\par
Next, let us check the convergence rates of \eqref{eq:SGD ODE flow Euclidean space} for general non-convex optimization or with PL inequality. It worth noting that for this problem, the convergence rate is measured by $\mE[\|\nabla F(\bx_{t})\|^{2}]$ or $\mE[F(\bx_{t})] - \inf_{\bx} F(\bx)$, with/without PL inequality.  
\begin{theorem}\label{thm:continuous sgd convergence}
	Let $\bx_{t}$ defined in \eqref{eq:SGD ODE flow Euclidean space}, then if $\mE\left[\tr\left(\Sigma_{\rm SGD}(\bx_{t})\nabla^{2} F(\bx_{t})\right)\right] \leq \sigma^{2}$ \footnote{This can be satisfied when $f_{\xi}(\bx)$ and its gradient are Lipschitz continuous. Because we have $\mE[\tr(\Sigma_{\rm SGD}(\bx_{t})\nabla^{2} F(\bx_{t}))] \leq \lambda_{\max}(\nabla^{2} F(\bx_{t}))\mE[\tr(\Sigma_{\rm SGD}(\bx_{t}))]$ due to the definition of $\Sigma_{\rm SGD}$.} by taking $\eta = \sqrt{\frac{\mE\left[F(\bx_{0}) - \inf_{\bx}F(\bx)\right]}{T\sigma^{2}}}$, we have 
	\begin{equation}
		\small
		\frac{1}{\eta T}\int_{0}^{\eta T}\mE_{\bx_{t}}\left[\left\|\nabla F(\bx_{t})\right\|^{2}\right]dt \leq \frac{2(F(\bx_{0}) - \inf_{\bx} F(\bx))}{\eta T}, 
	\end{equation}
	On the other hand, if $F(\bx_{t})$ satisfies PL inequality \eqref{eq:pl inequality}, by taking $\eta = 1 / \gamma T^{\alpha}$ with $0 < \alpha < 1$, we have  
	\begin{equation}
		\small
		\mE\left[F(\bx_{\eta T}) - \inf_{\bx}F(\bx)\right] \leq \frac{\sigma^{2}}{\gamma T^{\alpha}}.
	\end{equation}
	Besides that, if $F(\bx)$ is in the form of finite sum, the computational complexity is of order $\cO(\epsilon^{-2})$ to make $\mE_{\bx_{t}}[\|\nabla F(\bx_{t})\|^{2}]\leq \epsilon$ for some $t$, and $\cO(\epsilon^{-1/\alpha})$ to make $\mE_{\bx_{t}}[F(\bx_{t}) - \inf_{\bx}F(\bx)] \leq \epsilon$ under PL inequality.  
\end{theorem}
\begin{proof}
	Due to the definition of $\bx_{t}$, 
	\begin{equation}\label{eq:flow of F(xt)}
		\small
		\frac{\partial{F(\bx_{t})}}{\partial{t}} = \left\langle \nabla F(\bx_{t}), -\nabla F(\bx_{t}) - \frac{\eta}{2}\nabla\cdot\Sigma_{\rm SGD}(\bx_{t}) - \frac{\eta}{2}\Sigma_{\rm SGD}\nabla\log{\pi_{t}(\bx_{t})}\right\rangle.
	\end{equation}
	On the other hand, for $\bx_{t}\sim \pi_{t}$, 
	\begin{equation}
		\small
		\begin{aligned}
			\mE_{\bx_{t}}\left[\left\langle\nabla F(\bx_{t}), \Sigma_{\rm SGD}(\bx_{t})\nabla\log{\pi_{t}(\bx_{t})}\right\rangle\right] & = \int \left\langle\nabla F(\bx_{t}), \Sigma_{\rm SGD}(\bx_{t})\nabla\pi_{t}(\bx_{t})\right\rangle d\bx \\
			& = -\mE_{\bx_{t}}\left[\langle \nabla F(\bx_{t}), \nabla\cdot\Sigma_{\rm SGD}(\bx_{t})\rangle + \tr\left(\Sigma_{\rm SGD}(\bx_{t})\nabla^{2} F(\bx_{t})\right)\right].
		\end{aligned}
	\end{equation}
	Plugging this into \eqref{eq:flow of F(xt)}, we get 
	\begin{equation}\label{eq:sgd upper bound for F}
		\small
		\frac{\partial{\mE\left[F(\bx_{t})\right]}}{\partial{t}} = -\mE\left[\|\nabla F(\bx_{t})\|^{2}\right] + \frac{\eta}{2}\mE\left[\tr\left(\Sigma_{\rm SGD}(\bx_{t})\nabla^{2} F(\bx_{t})\right)\right] \leq -\mE\left[\|\nabla F(\bx_{t})\|^{2}\right] + \eta\sigma^{2}.
	\end{equation}
	Thus 
	\begin{equation}\label{eq:derivate SGD euclidean space}
		\small
		\begin{aligned}
			\frac{1}{\eta T}\int_{0}^{\eta T}\mE_{\bx_{t}}[\|\nabla F(\bx_{t})\|^{2}]dt & \leq \frac{\mE\left[F(\bx_{0}) - F(\bx_{T\eta})\right]}{\eta T} + \eta \sigma^{2} \\
			& = 2\sqrt{\frac{\mE\left[F(\bx_{0}) - F(\bx_{T\eta})\right]\sigma^{2}}{T}} \\
			& = \frac{2\mE\left[F(\bx_{0}) - F(\bx_{T\eta})\right]}{\eta T},
		\end{aligned}		
	\end{equation}
	by taking $\eta = \sqrt{\frac{\mE\left[F(\bx_{0}) - \inf_{\bx}F(\bx)\right]}{T\sigma^{2}}}$. Then we prove our first conclusion. 
	\par
	Next, let us consider the global convergence rate under PL inequality. By applying \eqref{eq:pl inequality} to \eqref{eq:sgd upper bound for F}, we get 
	\begin{equation}
		\small
		\frac{\partial{\mE\left[F(\bx_{t})\right] - \inf_{\bx}F(\bx)}}{\partial{t}} \leq -2\gamma(\mE\left[F(\bx_{t})\right] - \inf_{\bx}F(\bx)) + \eta\sigma^{2},
	\end{equation}
	which implies 
	\begin{equation}\label{eq:sgd exp convergence}
		\small
		\mE\left[F(\bx_{\eta T})\right] - \inf_{\bx}F(\bx) \leq e^{-2\gamma\eta T}\left(\mE\left[F(\bx_{0})\right] - \inf_{\bx}F(\bx)\right) + \eta\sigma^{2}\left(1 - e^{-2\gamma\eta T}\right),
	\end{equation}
	by Gronwall inequality. Thus by taking $\eta = 1 / \gamma T^{\alpha}$ with $0 < \alpha < 1$, we get second conclusion. 
	\par
	Under finite sum objective, when $\sqrt{T} = \cO(\eta T) = \cO(\epsilon^{-1})$, we have $\min_{0\leq t \leq \eta T}\mE_{\bx_{t}}[\|\nabla F(\bx_{t})\|^{2}]\leq \epsilon$. Thus, due to Proposition \ref{pro:equivalence of SGD}, it takes $M = \cO(\eta T/\eta) = \cO(\epsilon^{-2})$ steps in Algorithm \ref{alg:discrete riemannian sgd} to get the ``$\epsilon$-stationary point''. 
	\par
	Furthermore, under PL inequality, due to \eqref{eq:sgd exp convergence}, it takes $T = \cO(\epsilon^{-1/\alpha})$ steps to make $\mE\left[F(\bx_{\eta T})\right] - \inf_{\bx}F(\bx) \leq \epsilon$, which results in computational complexity of order $\cO(\epsilon^{-1/\alpha})$.   
\end{proof}

The convergence rate of Riemannian SGD flow can be similarly proven as in the above theorem. Next, let us check it. We need a lemma termed as von Neumann's trace inequality \citep{von1937some} to prove Theorem \ref{thm:continuous sgd convergence}.
\begin{lemma}[von Neumann's trace inequality]\label{lem:trace inequality}
	For systematic matrices $\bA$ and $\bB$, let $\{\lambda_{i}(\bA)\}$ and $\{\lambda_{i}(\bB)\}$ respectively be their eigenvalues with descending orders. Then 
	\begin{equation}
		\small
		\tr(\bA\bB) \leq \sum\limits_{i=1}^{n}\lambda_{i}(\bA)\lambda_{i}(\bB).
	\end{equation}
\end{lemma}

\convergenceofriemanniansgd*
\begin{proof}
	During the proof, we borrow the notations of $H_{1}, H_{2}, H_{3}$ in Lemma \ref{lemma:regularity bound}. Next, we prove the results as in Theorem \ref{thm:continuous sgd convergence}. That is 
	\begin{equation}\label{eq:sgd riemannian derivate}
		\small
		\begin{aligned}
			\frac{\partial}{\partial{t}}D_{KL}(\pi_{t}\parallel \mu) & = \left\langle\grad D_{KL}(\pi_{t}\parallel \mu), -\grad D_{KL}(\pi_{t}\parallel \mu) + \frac{\eta}{2}\nabla\cdot \Sigma_{\rm SGD} + \frac{\eta}{2}\Sigma_{\rm SGD}\nabla\log{\pi_{t}}\right\rangle_{\pi_{t}} \\
			& = -\left\|\grad D_{KL}(\pi_{t}\parallel \mu)\right\|^{2} + \left\langle\grad D_{KL}(\pi_{t}\parallel \mu), \frac{\eta}{2}\nabla\cdot \Sigma_{\rm SGD} + \frac{\eta}{2}\Sigma_{\rm SGD}\nabla\log{\pi_{t}}\right\rangle_{\pi_{t}}. 
		\end{aligned}
	\end{equation}
	To begin with, we have
	\begin{equation}\label{eq:bound 1 sgd}
		\small
		\begin{aligned}
			\left\langle\grad D_{KL}(\pi_{t}\parallel \mu), \frac{\eta}{2}\nabla\cdot \Sigma_{\rm SGD} \right\rangle_{\pi_{t}} & =\int \left\langle\nabla  \log \frac{d\pi_t}{d\mu}, \frac{\eta}{2}\nabla\cdot \Sigma_{\rm SGD} \right\rangle d \pi_t
			\\
			& \leq \frac{\eta H_{3}}{2}\|\grad D_{KL}(\pi_{t}\parallel \mu)\|^{2}_{\pi_{t}} + \frac{\eta}{8H_{3}}\mE_{\pi_{t}}\left[\left\|\nabla\cdot \Sigma_{\rm SGD}\right\|^{2}\right]
			\\
			& \leq\frac{\eta H_{3}}{2}\|\grad D_{KL}(\pi_{t}\parallel \mu)\|^{2}_{\pi_{t}} + \frac{\eta H_{2}^{2}}{8H_{3}}, 
		\end{aligned}
	\end{equation}
	where the first inequality is due to Young's inequality and last one is from the Lemma \ref{lemma:regularity bound}. Then, we have  
	\begin{equation}\label{eq:bound 2 sgd}
		\small
		\begin{aligned}
			&\left\langle\grad D_{KL}(\pi_{t}\parallel \mu), \frac{\eta}{2}\Sigma_{\rm SGD}\nabla\log{\pi_{t}} \right\rangle_{\pi_{t}} 
			= \int \left\langle\nabla  \log \frac{d \pi_t}{d \mu}, \frac{\eta}{2}\Sigma_{\rm SGD}\nabla\log{\pi_{t}} \right\rangle d \pi_t
			\\
			=& \int \left\langle\nabla  \log \pi_t, \frac{\eta}{2}\Sigma_{\rm SGD}\nabla\log \pi_{t} \right\rangle d \pi_t- \int \left\langle\nabla  \log { \mu}, \frac{\eta}{2}\Sigma_{\rm SGD}\nabla\log{\pi_t} \right\rangle d \pi_t
			\\
			\le &  \frac{\eta H_{3}}{2}\int \left\|\nabla  \log d \pi_t\right\|^{2} d \pi_t- \int \left\langle\nabla  \log { \mu}, \frac{\eta}{2}\Sigma_{\rm SGD}\nabla\log{\pi_t} \right\rangle d \pi_t
			\\
			= & \frac{\eta H_{3}}{2} \left\|\grad D_{KL}(\pi_{t}\parallel \mu)\right\|_{\pi_{t}}^{2} - \int \left\langle\nabla  \log { \mu}, \frac{\eta}{2}\Sigma_{\rm SGD}\nabla\log{\pi_t} \right\rangle d \pi_t
			\\
			-& \frac{\eta H_{3}}{2}\int \left\|\nabla  \log  \mu\right\|^{2}d \pi_t + \frac{\eta H_{3}}{2}\int \left\langle\nabla  \log { \mu}, \frac{\eta}{2}\nabla\log{\pi_t} \right\rangle d \pi_t.
		\end{aligned}
	\end{equation}
	In the r.h.s of the above inequality, the sum of the second and the forth terms can be bounded as 
	\begin{equation}\label{eq:bound 3 sgd}
		\small
		\begin{aligned}
			-\int&  \left\langle\nabla  \log { \mu}, \frac{\eta}{2}\Sigma_{\rm SGD}\nabla\log{\pi_t} \right\rangle d \pi_t+\frac{\eta H_{3}}{2}\int \left\langle\nabla  \log { \mu}, \frac{\eta}{2}\nabla\log{\pi_t} \right\rangle d \pi_t \\
			& = -\int \left\langle\nabla  \log { \mu}, \frac{\eta}{2}\Sigma_{\rm SGD}\nabla\log{\pi_t} \right\rangle d \pi_t+\frac{\eta H_{3}}{2}\int \left\langle\nabla  \log { \mu}, \frac{\eta}{2}\nabla\log{\pi_t} \right\rangle d \pi_t
			\\
			&- \int \left\langle\nabla  \log  \mu, \frac{\eta}{2}\nabla\cdot \Sigma_{\rm SGD} \right\rangle d \pi_t + \int \left\langle\nabla  \log  \mu, \frac{\eta}{2}\nabla\cdot \Sigma_{\rm SGD} \right\rangle d \pi_t
			\\
			& =\frac{\eta}{2}\int\tr\left(\nabla^{2}\log{\mu}\Sigma_{\rm SGD}\right)d\pi_{t}- \frac{\eta H_{3}}{2}\int\tr\left(\nabla^{2}\log{\mu}\right)d\pi_{t}
			\\
			& + \int \left\langle\nabla  \log  \mu, \frac{\eta}{2}\nabla\cdot \Sigma_{\rm SGD} \right\rangle d \pi_t
			\\
			& \overset{a}{\leq} \frac{\eta H_{1}H_{3}}{2}+\frac{\eta  H_{1}H_{3}}{2}+ \frac{\eta H_{3}}{2}\int \left\Vert \nabla  \log  \mu \right\Vert  d \pi_t
			\\
			& \le \frac{\eta H_{1}H_{3}}{2}+\frac{\eta H_{1}H_{3}}{2}+ \frac{\eta H_{3}}{2}\int \left\Vert \nabla  \log  \mu \right\Vert^2  d \pi_t + \frac{\eta H_{2}^{2}}{8H_{3}},
		\end{aligned}
	\end{equation}
	where the inequality $a$ is from Lemma \ref{lem:trace inequality}, \ref{lemma:regularity bound}, and the semi-positive definite property of $\Sigma_{\rm SGD}$. By plugging \eqref{eq:bound 1 sgd}, \eqref{eq:bound 2 sgd}, and \eqref{eq:bound 3 sgd} into \eqref{eq:sgd riemannian derivate}, we have 
	\begin{equation}
		\small
		\frac{\partial}{\partial{t}}D_{KL}(\pi_{t}\parallel \mu) \leq -(1 - \eta H_{3})\left\|\grad D_{KL}(\pi_{t}\parallel \mu)\right\|^{2}_{\pi_{t}} + \eta H_{1}H_{3} + \frac{\eta H_{2}^{2}}{4H_{3}},
	\end{equation}
	due to the value of $\eta$ makes $\eta H_{3} \leq \frac{1}{2}$. Taking integral w.r.t. $t$ as in \eqref{eq:derivate SGD euclidean space} implies 
	\begin{equation}\label{eq:riemannian sgd upper bound}
		\small
		\begin{aligned}
			\frac{1}{2\eta T}\int_{0}^{\eta T}\left\|\grad D_{KL}(\pi_{t}\parallel \mu)\right\|^{2}_{\pi_{t}}dt & \leq \frac{D_{KL}(\pi_{0}\parallel \mu)}{\eta T} + \eta H_{1}H_{3} + \frac{\eta H_{2}^{2}}{4H_{3}}\\
			& = 2\sqrt{\frac{D_{KL}(\pi_{0}\parallel \mu)}{T}\left(H_{1}H_{3} + \frac{H_{2}^{2}}{4H_{3}}\right)} \\
			& = 2D_{KL}(\pi_{0}\parallel \mu)\frac{1}{\eta T},
		\end{aligned}
	\end{equation}
	so that we prove our conclusion due to the value of $H_{1}, H_{2}, H_{3}$. The second result can be similarly obtained as in Theorem \ref{thm:continuous sgd convergence} by inequality 
	\begin{equation}
		\small
		D_{KL}(\pi_{\eta T}\parallel \mu) \leq e^{-\gamma \eta T}D_{KL}(\pi_{0}\parallel \mu) + \eta\left(H_{1}H_{3} + \frac{H_{2}^{2}}{4H_{3}}\right)\left(1 - e^{-\gamma\eta T}\right),
	\end{equation}
	which is obtained by applying log-Sobolev inequality to \eqref{eq:riemannian sgd upper bound} and Gronwall inequality. Thus, taking $\eta = 1 / \gamma T^{\alpha}$ with $0 < \alpha < 1$ implies our conclusion. 
\end{proof}
Similar to the proof of Theorem \ref{thm:convergence riemannian sgd}, we can get the computational complexity of Riemannian GD and SGD flow, i.e., $\cO(N\epsilon^{-1})$ and $\cO(\epsilon^{-2})$ respectively for non-convex problem to arrive $\epsilon$-stationary point, but $\cO(N\gamma^{-1}\log{\epsilon^{-1}})$ and $\cO(\epsilon^{-1})$ under (log-Sobolev inequality) Riemannian PL inequality.   
\par
The following is the lemma implied by Assumption \ref{ass:stochastic continuous}. 
\begin{lemma}\label{lemma:regularity bound}
	Under Assumption \ref{ass:stochastic continuous}, it holds $\mE_{\pi_{t}}[\tr^{+}(\nabla^{2}\log{\mu})] \le dL_{2} = H_{1}$ \footnote{ $\tr^+(*)$  stands for the sum of absolute value of $*$'s eigenvalues.}, $\sup_{\bx} \Vert \nabla \cdot \Sigma_{\rm SGD} \Vert \le 4(d + 1)L_{1}L_{2} = H_2$, and $\sup_{\bx}\lambda_{\max}(\Sigma_{\rm SGD}(\bx)) \leq 4L_{1}^{2} = H_{3}$.
\end{lemma}
\begin{proof}
	Due to Assumption \ref{ass:stochastic continuous}, we have 
	\begin{equation}\label{eq:positive trace bound}
		\small
		\begin{aligned}
			\mE_{\pi_{t}}[\tr^{+}(\nabla^{2}\log{\mu})] &=\mE_{\pi_{t}}\left[\tr^{+}\left(\mE_{\xi}\left[\nabla^{2}\log{\mu_{\xi}}\right]\right)\right] \\
			& \leq \mE_{\pi_{t}}\left[d\left\|\mE_{\xi}\nabla^{2}\log{\mu_{\xi}}\right\|\right] \\
			& \leq dL_{2},
		\end{aligned}
		\small
	\end{equation}
	where $\|\cdot\|$ here is the spectral norm of matrix. On the other hand, we notice 
	\begin{equation}
		\small
		\begin{aligned}
			\nabla\cdot \Sigma_{\rm SGD} &= \mE_{\xi}\left[\tr\left(\nabla^{2}\log{\mu_{\xi}} -  \mE_{\xi}\left[\nabla^{2}\log{\mu_{\xi}}\right]\right)\left(\nabla\log{\mu_{\xi}} -  \mE\left[\nabla\log{\mu_{\xi}}\right]\right)\right] \\
			& + \mE_{\xi}\left[\left(\nabla^{2}\log{\mu_{\xi}} -  \mE_{\xi}\left[\nabla^{2}\log{\mu_{\xi}}\right]\right)\left(\nabla\log{\mu_{\xi}} -  \mE\left[\nabla\log{\mu_{\xi}}\right]\right)\right].
		\end{aligned}
	\end{equation}
	Then by Assumption \ref{ass:stochastic continuous}, 
	\begin{equation}\label{eq:upper bound div of covariance matrix}
		\small
		\sup_{\bx}\|\nabla\cdot \Sigma_{\rm SGD}(\bx)\| \leq 4dL_{1}L_{2} + 4L_{1}L_{2} = 4(d + 1)L_{1}L_{2}. 
	\end{equation}
	Finally, due to Assumption \ref{ass:stochastic continuous}, we have 
	\begin{equation}\label{eq:specturm norm of covariance matrix sgd}
		\small
		\sup_{\bx}\lambda_{\max}(\Sigma_{\rm SGD}(\bx)) \leq 4L_{1}^{2}.
	\end{equation}
\end{proof}

\section{Proofs in Section \ref{sec:svrg flow}}\label{app:proofs in section svrg flow}
\svrgflow*
\begin{proof}
	The proof is similar to the one of Proposition \ref{pro:equivalence of SGD}. Firstly, we choose $\Gamma_{\pi_{0}^{i}}^{\pi_{n}^{i}}$ as the  transportation that preserves the correlation between $\bx_{0}^{i}$ and $\bx_{n}^{i}$, which means that for any function $\bu\in\cT_{\pi_{0}^{i}}$, we have 
	\begin{equation}
		\small
		\Gamma_{\pi_{0}^{i}}^{\pi_{n}^{i}}(\bu(\bx_{0}^{i})) = \bu(\bx_{n}^{i}),
	\end{equation}
	so that $\pi_{0}^{i}\times \pi_{n}^{i}$ is the union distribution of $(\bx_{0}^{i}, \bx_{n}^{i})$ defined as below. Concretely, we know that for $0 \leq n \leq M - 1, 0\leq i \leq I - 1$, the corresponded $\bx_{n}^{i}$ of discrete SVRG in Algorithm \ref{alg:discrete riemannian svrg} satisfies 
	\begin{equation}
		\small
		\bx_{n + 1}^{i} = \bx_{n}^{i} + \eta \left(\nabla\log{\frac{d\mu_{\xi_{n}^{i}}}{d\pi_{n}^{i}}}(\bx_{n}^{i}) - \nabla\log{\frac{d\mu_{\xi_{n}^{i}}}{d\pi_{0}^{i}}}(\bx_{0}^{i}) + \nabla\log{\frac{d\mu}{d\pi_{0}^{i}}}(\bx_{0}^{i})\right).
	\end{equation}
	In the rest of this proof, we neglect the subscript $i$ to simplify the notations. Similar to Proposition \ref{pro:equivalence of SGD}, we can prove $\bx_{n + 1}$ is approximated by 
	\begin{equation}
		\small
		\barx_{n + 1} = \barx_{n} + \eta\left(\nabla\log{\mu_{\xi_{n}}}(\barx_{n}) - \nabla\log{\mu_{\xi_{n}}}(\bx_{0}) + \nabla\log{\mu(\bx_{0})}\right) + \sqrt{2\eta}\beps_{n},
	\end{equation}
	with $\beps_{n}\sim\cN(0, \bI)$ by noting the Lipchitz continuity of $\nabla\log{\mu_{\xi}}$. As in Proposition \ref{pro:equivalence of SGD}, the approximation error is similarly proven as   
	\begin{equation}\label{eq:svrg discrete gap}
		\small
		\mE\left[\|\barx_{n} - \bx_{n}\|^{2}\mid \bx_{0}, \barx_{0}\right] \leq \cO(\eta) + \cO(\|\bx_{0} - \barx_{0}\|^{2}). 
	\end{equation}
	Next, our goal is showing that the discrete dynamics $\barx_{n}$ is approximated by the following stochastic differential equation 
	\begin{equation}\label{eq:svrg sde}
		\small
		\begin{aligned}
			d\hbx_{t} &= \nabla \mE_{\xi}\left[\log\mu_{\xi}(\hbx_{t})\right]dt + \left(\sqrt{\eta}\Sigma_{\rm SVRG}^{\frac{1}{2}}(\hbx_{t_{i}}, \hbx_{t}), \sqrt{2}\bI\right)dW_{t} \\
			& = \bb(\hbx_{t})dt + \left(\sqrt{\eta}\Sigma_{\rm SVRG}^{\frac{1}{2}}(\hbx_{t_{i}}, \hbx_{t}), \sqrt{2}\bI\right)dW_{t}; \qquad iM\eta = t_{i} \leq t \leq t_{i + 1} = (i + 1)M\eta,
		\end{aligned}	
	\end{equation}
	where $M$ is steps for each epoch, and $\Sigma_{\rm SVRG}^{\frac{1}{2}}(\hbx_{t_{i}}, \bx_{t})$ is defined in \eqref{eq:svrg variance}. Similar to \eqref{eq:continuous expansion on sde}, for $(n + iM)\eta = t\in [t_{i}, t_{i + 1}]$. We write $\hbx_{(n + iM)\eta}$ as $\hbx_{n\eta}$ in the rest of this proof to simplify the notation. We can prove that for any $f\in C^{2}$ with bounded gradient and Hessian under the condition of given $\hbx_{t_{i}}$,
	\begin{equation}
		\small
		\begin{aligned}
			\mE^{\hbx_{t_{i}}}\left[f(\hbx_{(n + 1)\eta})\right] & = \mE^{\hbx_{t_{i}}}\left[f(\hbx_{n\eta})\right] + \eta\mE^{\hbx_{t_{i}}}[\langle \bb(\hbx_{n\eta}), \nabla f(\hbx_{n\eta})\rangle] + \eta\mE^{\hbx_{t_{i}}}\left[\Delta f(\hbx_{n\eta})\right] \\
			& + \frac{\eta^{2}}{2}\mE^{\hbx_{t_{i}}}\left[\tr\left[\left(\Sigma_{\rm SVRG}(\hbx_{t_{i}}, \hbx_{n\eta})\right)\nabla^{2}f(\hbx_{n\eta})\right]\right] + \cO(\eta^{2})
		\end{aligned}
	\end{equation}
	On the other hand 
	\begin{equation}
		\small
		\begin{aligned}
			& \mE^{\hbx_{t_{i}}}\left[f\left(\hbx_{n\eta} + \eta\left(\nabla\log{\mu_{\xi_{n}}(\hbx_{n\eta})} - \nabla\log{\mu_{\xi_{n}}(\hbx_{t_{i}})} + \nabla\log{\mu(\hbx_{t_{i}})}\right)\right) + \sqrt{2\eta}\beps_{n}\right] \\
			& = \mE^{\hbx_{t_{i}}}\left[f(\hbx_{n\eta})\right] + \eta\mE^{\hbx_{t_{i}}}[\langle \bb(\hbx_{n\eta}), \nabla f(\hbx_{n\eta})\rangle] +\eta\mE^{\hbx_{t_{i}}}\left[\Delta f(\hbx_{n\eta})\right]\\
			& + \frac{\eta^{2}}{2}\mE^{\hbx_{t_{i}}}\left[\tr\left[\left(\Sigma_{\rm SVRG}(\hbx_{t_{i}}, \hbx_{n\eta}) + \bb\bb^{\top}\right)\nabla^{2}f(\hbx_{n\eta})\right]\right] + \cO(\eta^{3}),
		\end{aligned}
	\end{equation}
	where we use the fact that 
	\begin{equation}
		\small
		\mE^{\hbx_{t_{i}}}\left[\nabla\log{\mu_{\xi_{n}}(\hbx_{n\eta})} - \nabla\log{\mu_{\xi_{n}}(\hbx_{t_{i}})} + \nabla\log{\mu(\hbx_{t_{i}})}\right] = \mE^{\hbx_{t_{i}}}\left[\nabla\log{\mu_{\xi_{n}}(\hbx_{n\eta})}\right] = \bb(\hbx_{n\eta}).
	\end{equation}
	Then similar to \eqref{eq:riemannian sgd discrete gap} in the proof of Proposition \ref{pro:equivalence of SGD}, by taking $f$ as $f_{\delta, C}$ in \eqref{eq:f_delta_C}, and combining \eqref{eq:svrg discrete gap}, we prove 
	\begin{equation}
		\small
		\mE\left[\left\|\hbx_{n\eta} - \bx_{n}\right\|^{2} \mid \hbx_{t_{i}}, \bx_{0} \right] \leq \cO(\eta) + \cO\left(\|\hbx_{t_{i}} - \bx_{0}\|^{2}\right).
	\end{equation}
	By noting that $\hbx_{0} = \bx_{0}^{0}$, recursively using the above inequality over $i$, and taking expectation, we prove that 
	\begin{equation}
		\small
			\mE\left[\left\|\hbx_{(iM + n)\eta} - \bx_{n}^{i}\right\|^{2}\right] \leq \cO(\eta). 
	\end{equation}
	On the other hand, we know that given $\hbx_{t_{i}} = \by$, the conditional probability $\hbx_{t}\mid \hbx_{t_{i}}$ follows conditional density  $\pi_{t\mid t_{i}}(\bx)$ satisfies 
	\begin{equation}
		\small
		\frac{\partial}{\partial{t}}\pi_{t\mid t_{i}}(\bx\mid \by) = \nabla\cdot\left[\pi_{t\mid t_{i}}(\bx\mid \by)\left(\nabla \log{\frac{d\pi_{t}}{d\mu}}(\bx) - \frac{\eta}{2}\nabla_{\bx}\cdot\Sigma_{\rm SVRG}(\by, \bx) - \frac{\eta}{2}\Sigma_{\rm SVRG}(\by, \bx)\nabla_{\bx}\log{\pi_{t_{i}, t}(\by, \bx)}\right)\right].
	\end{equation}
	Then, multiplying $\pi_{t_{i}}(\by)$ to the above equality, applying equality $\nabla_{\bx}\log{\pi_{t_{i}, t}(\by, \bx)} = \nabla_{\bx}\log{\pi_{t\mid t_{i}}(\bx\mid \by)}$, and taking integral over $\by$ implies our conclusion. 
\end{proof}
\subsection{Convergence of Riemannian SVRG Flow}\label{app:convergence of riemannian svrg flow}
Similar to the results of Riemannian SGD in Section \ref{app:convergence sgd}, we first give the convergence rate of continuous SVRG flow (stochastic ODE) in the Euclidean space, which helps understanding the results in Wasserstein space. First, to minimize $F(\bx) = 1/N\sum_{j=1}^{N}f_{\xi_{j}}(\bx)$, by generalizing the formulation in \eqref{eq:svrg flow on manifold}, the SVRG flow in Euclidean space is   
\begin{equation}\label{eq:svrg ode flow in euclidean space}
	\small
	d\bx_{t} = -\nabla F(\bx_{t}) - \frac{\eta}{2}\nabla\cdot\Sigma_{\rm SVRG}(\bx_{t_{i}}, \bx_{t}) - \frac{\eta}{2}\Sigma_{\rm SVRG}(\bx_{t_{i}}, \bx_{t})\nabla\log{\pi_{t}(\bx_{t})}dt, \qquad t_{i} \leq t \leq t_{i + 1};
\end{equation}
with $\Sigma_{\rm SVRG}$ defined as
\begin{equation}\label{eq:svrg variance euclidean space}
	\small
	\Sigma_{\rm SVRG}(\by, \bx) = \mE_{\xi, \bx_{t_{i}}}\left[(\nabla f_{\xi}(\bx) - \nabla f_{\xi}(\by) + \nabla F(\by) - \nabla F(\bx))(\nabla f_{\xi}(\bx) - \nabla f_{\xi}(\by) + \nabla F(\by) - \nabla F(\bx))^{\top}\right],
\end{equation} 
and $\pi_{t}$ is the density of $\bx_{t}$, $\xi$ is uniform distribution over $\{\xi_{i}\}$. Then the convergence rate and computational complexity of \eqref{eq:svrg ode flow in euclidean space} is presented in the following Theorem.  
\begin{theorem}\label{thm:convergence of svrg on euclidean}
	For $\bx_{t}$ in \eqref{eq:svrg ode flow in euclidean space}, learning rate $\eta = \cO(N^{-2/3})$, $\Delta = t_{1} - t_{0} = \cdots = t_{I} - t_{I -1} = \cO(1 / \sqrt{\eta})$, and $\eta T = I\Delta$, if $f_{\xi}(\bx)$ is Lipschitz continuous with coefficient $L_{1}$ and has Lipschitz continuous gradient with coefficient $L_{2}$, then 
	\begin{equation}
		\small
		\frac{1}{\eta T}\sum\limits_{i = 1}^{I}\int_{t_{i}}^{t_{i + 1}}\mE\left[\|\nabla F(\bx_{t})\|^{2}\right]dt \leq \frac{2(\mE\left[F(\bx_{0})\right] - \inf_{\bx} F(\bx))}{\eta T},
	\end{equation}
	On the other hand, by properly taking hyperparameters, the computational complexity of SVRG flow is of order $\cO(N^{2/3}/\epsilon)$, when $\min_{0\leq t \leq \eta T} \mE\left[\|\nabla F(\bx_{t})\|^{2}\right] \leq \epsilon$. Further more, when $F(\bx)$ satisfies PL inequality \eqref{eq:pl inequality} with coefficient $\gamma$, we have 
	\begin{equation}
		\small
		\mE\left[F(\bx_{\eta T}) - \inf F(\bx)\right] \leq e^{-\gamma \eta T}\mE\left[F(\bx_{0}) - \inf_{\bx}F(\bx)\right].
	\end{equation} 
	The computational complexity of SVRG flow is of order $\cO((N + \gamma^{-1}N^{2/3})\log{\epsilon^{-1}})$ to make $\mE\left[F(\bx_{\eta T}) - \inf F(\bx)\right] \leq \epsilon$. 
\end{theorem}
\begin{proof}
	Let us define the Lyapunov function, for $t\in [t_{i}, t_{i + 1}]$ (with out loss of generality, let $i = 0$), 
	\begin{equation}
		\small
		R_{t}(\bx) = F(\bx) + \frac{c_{t}}{2}\|\bx - \bx_{t_{0}}\|^{2}. 
	\end{equation}
	Then we have for $\bx_{t}$ defined in \eqref{eq:svrg ode flow in euclidean space}
	\begin{equation}\label{eq:svrg drt}
		\small
		d{R_{t}(\bx_{t})} = \left\langle\nabla F(\bx_{t}), d\bx_{t}\right\rangle + \frac{c_{t}^{\prime}}{2}\|\bx_{t} - \bx_{t_{0}}\|^{2}dt + c_{t}\left\langle\bx_{t} - \bx_{t_{0}}, d\bx_{t}\right\rangle.
	\end{equation}
	Due to \eqref{eq:svrg ode flow in euclidean space}, $\mE\left[\|\bx - \mE[\bx]\|^{2}\right] \leq \mE\left[\|\bx\|^{2}\right]$, and Lemma \ref{lem:trace inequality} we have 
	\begin{equation}\label{eq:svrg ode f descent}
		\small
		\begin{aligned}
			\mE& \left[\left\langle\nabla F(\bx_{t}), \frac{d\bx_{t}}{dt}\right\rangle\right] = -\mE\left[\|\nabla F(\bx_{t})\|^{2}\right] + \frac{\eta}{2}\mE\left[\tr\left(\Sigma_{\rm SVRG}(\bx_{t_{0}}, \bx_{t})\right)\nabla^{2} F(\bx_{t})\right] \\
			& \leq -\mE\left[\|\nabla F(\bx_{t})\|^{2}\right] + \frac{\eta\lambda_{\max}(\nabla^{2} F(\bx_{t}))}{2}\mE\left[\tr(\Sigma_{\rm SVRG}(\bx_{t_{0}}, \bx_{t}))\right] \\
			& = -\mE\left[\|\nabla F(\bx_{t})\|^{2}\right] + \frac{\eta\lambda_{\max}(\nabla^{2} F(\bx_{t}))}{2}\mE\left[\left\|\nabla f_{\xi}(\bx_{t}) - \nabla f_{\xi}(\bx_{t_{0}}) + \nabla F(\bx_{t_{0}}) - \nabla F(\bx_{t})\right\|^{2}\right] \\
			& \leq -\mE\left[\|\nabla F(\bx_{t})\|^{2}\right] + \frac{\eta L_{2}}{2}\mE\left[\left\|\nabla f_{\xi}(\bx_{t}) - \nabla f_{\xi}(\bx_{t_{0}})\right\|^{2}\right] \\
			& \leq -\mE\left[\|\nabla F(\bx_{t})\|^{2}\right] + \frac{\eta L_{2}^{2}}{2}\mE\left[\left\|\bx_{t} - \bx_{t_{0}}\right\|^{2}\right].
		\end{aligned}
	\end{equation}
	On the other hand, by Young's inequality, for some $\beta > 0$, 
	\begin{equation}\label{eq:svrg ode distance descent}
		\small
		\begin{aligned}
			\mE\left[\left\langle\bx_{t} - \bx_{t_{0}}, \frac{d\bx_{t}}{dt}\right\rangle\right] & = \mE\left[-\langle\bx_{t} - \bx_{t_{0}}, \nabla F(\bx_{t})\rangle\right] + \frac{\eta}{2}\mE\left[\tr\left(\Sigma_{\rm SVRG}(t_{0}, \bx_{t})\right)\right] \\
			& \leq \frac{\beta}{2}\mE\left[\|\bx_{t} - \bx_{t_{0}}\|^{2}\right] + \frac{1}{2\beta}\mE\left[\left\|\nabla F(\bx_{t})\right\|^{2}\right] + \frac{\eta L_{2}}{2}\mE\left[\|\bx_{t} - \bx_{t_{0}}\|^{2}\right].
		\end{aligned}
	\end{equation}
	By plugging \eqref{eq:svrg ode f descent} and \eqref{eq:svrg ode distance descent} into \eqref{eq:svrg drt}, we have 
	\begin{equation}\label{eq:svrg ode drt upper bound}
		\small
		\frac{\partial}{\partial{t}}\mE\left[R_{t}(\bx_{t})\right] \leq -\left(1 - \frac{c_{t}}{2\beta}\right)\mE\left[\|\nabla F(\bx_{t})\|^{2}\right] + \left[\frac{(\beta + \eta L_{2})c_{t}}{2} + \frac{\eta L_{1}L_{2}}{2} + \frac{c_{t}^{\prime}}{2}\right]\mE\left[\|\bx_{t} - \bx_{t_{0}}\|^{2}\right].
	\end{equation}
	By making 
	\begin{equation}\label{eq:ct ode}
		\small
		(\beta + \eta L_{2})c_{t} + \eta L_{1}L_{2} + c_{t}^{\prime} = 0, 
	\end{equation}
	we get 
	\begin{equation}
		\small
		c_{t_{1}} = e^{-(\beta + \eta L_{2})(t_{1} - t_{0})}c_{t_{0}} - \frac{\eta L_{1}L_{2}}{\beta + \eta L_{2}}\left(1 - e^{-(\beta + \eta L_{2})(t_{1} - t_{0})}\right). 
	\end{equation}
	By taking $c_{t_{0}} = \sqrt{\eta}, c_{t_{1}} = 0, \beta = \sqrt{\eta}$ we get 
	\begin{equation}\label{eq:svrg delta t}
		\small
		t_{1} - t_{0} = \frac{\log\left({1 + \sqrt{\eta} / L_{2} + 1 / L_{1}L_{2}}\right)}{\sqrt{\eta} + \eta L_{2}} \leq \cO\left(\frac{1}{\sqrt{\eta}}\right),
	\end{equation}
	when $\eta \to 0$. On the other hand, due to $c_{t}^{\prime} < 0$ and \eqref{eq:ct ode}, we have 
	\begin{equation}\label{eq:descent lemma for svrg ode}
		\small
		\begin{aligned}
			\frac{1}{2(t_{1} - t_{0})}\int_{t_{0}}^{t_{1}}\mE\left[\|\nabla F(\bx_{t})\|^{2}\right]dt 
			& = \left(\frac{1}{t_{1} - t_{0}}\right)\int_{t_{0}}^{t_{1}}\left(1 - \frac{c_{t_0}}{2\beta}\right)\mE\left[\|\nabla F(\bx_{t})\|^{2}\right]dt\\
			& \leq \left(\frac{1}{t_{1} - t_{0}}\right)\int_{t_{0}}^{t_{1}}\left(1 - \frac{c_{t}}{2\beta}\right)\mE\left[\|\nabla F(\bx_{t})\|^{2}\right]dt \\
			& \leq \frac{\mE\left[R_{t_{0}}(\bx_{t_{0}})\right] - \mE\left[R_{t_{1}}(\bx_{t_{1}})\right]}{t_{1} - t_{0}} \\
			& = \frac{\mE\left[F(\bx_{t_{0}}) - F(\bx_{t_{1}})\right]}{t_{1} - t_{0}}. 
		\end{aligned}
	\end{equation}
	Thus, for any $T = t_{I}$ and $\Delta = t_{1} - t_{0} = \cdots = t_{I} - t_{I - 1}$ defined in \eqref{eq:svrg delta t}, we have 
	\begin{equation}
		\small
		\frac{1}{I\Delta}\sum\limits_{i = 1}^{I}\int_{t_{i}}^{t_{i + 1}}\mE\left[\|\nabla F(\bx_{t})\|^{2}\right]dt \leq \frac{2(\mE\left[F(\bx_{t})\right] - \inf_{\bx} F(\bx))}{T},
	\end{equation} 
	which leads to the required convergence rate. Next, let us check the computational complexity of it. Due to $I\Delta = \eta T = \cO(\epsilon^{-1})$, $\Delta = \cO(1 / \sqrt{\eta})$, and Proposition \ref{pro:riemannian svrg flow}, we should running the SVRG for $I = \cO(\sqrt{\eta} / \epsilon)$ epochs with $M = \Delta / \eta$ steps in each epoch in Algorithm \ref{alg:discrete riemannian svrg}. Besides, note that for each epoch, the computational complexity is of order $\cO(M + N) = \cO(\Delta / \eta + N)$. Then, by taking $\eta = \cO(N^{-2/3})$, the computational complexity of SVRG flow is of order 
	\begin{equation}
		\small
		I\cO(M + N)= I\cO(\Delta / \eta + N) = \cO\left(\frac{\sqrt{\eta}}{\epsilon}\left(\eta^{-\frac{3}{2}} + N\right)\right) = \cO\left(\frac{N^{\frac{2}{3}}}{\epsilon}\right). 
	\end{equation}
	\par
	Next, let us check the results when $F(\bx)$ satisfies PL inequality with coefficient $\gamma$. We will follow the above notations in the rest of this proof. For any $t \in [t_{0}, t_{1}]$, we can reconstruct the $c_{t}$ in Lyapunov function \eqref{eq:svrg drt} such that 
	\begin{equation}
		\small
		(2\gamma + \beta + \eta L_{2})c_{t} + \eta L_{1}L_{2} + c_{t}^{\prime} = 0, 
	\end{equation}
	which implies 
	\begin{equation}
		\small
		c_{t_{1}} =  e^{-(2\gamma + \beta + \eta L_{2})(t_{1} - t_{0})}c_{t_{0}} - \frac{\eta L_{1}L_{2}}{2\gamma + \beta + \eta L_{2}}\left(1 - e^{-(2\gamma + \beta + \eta L_{2})(t - t_{0})}\right). 
	\end{equation}
	Similarly, by taking $c_{t_{0}} = \sqrt{\eta}, c_{t_{1}} = 0, \beta = \sqrt{\eta}$, we get
	\begin{equation}\label{eq:svrg pl delta t}
		\small
		t_{1} - t_{0} = \frac{\log\left({1 + 2\gamma/\sqrt{\eta}L_{1}L_{2} + \sqrt{\eta} / L_{2} + 1 / L_{1}L_{2}}\right)}{2\gamma + \sqrt{\eta} + \eta L_{2}} = \min\left\{\cO\left(\frac{1}{\sqrt{\eta}}\right), \cO\left(\frac{1}{\gamma}\right)\right\} ,
	\end{equation}
	when $\eta \to 0$. Plugging this into \eqref{eq:svrg ode drt upper bound}, combining PL inequality and the monotonically decreasing property of $c_{t}$, we have
	\begin{equation}
		\small
		\begin{aligned}
			\frac{\partial}{\partial{t}}\mE\left[R_{t}(\bx_{t}) - \inf_{\bx}F(\bx)\right] & \leq -2\gamma\left(1 - \frac{c_{t}}{2\beta}\right)\left(\mE\left[F(\bx_{t})\right] - \inf_{\bx}F(\bx)\right) - \gamma c_{t}\mE\left[\|\bx_{t} - \bx_{t_{0}}\|^{2}\right] \\
			& \leq \gamma\left(\mE\left[F(\bx_{t})\right] - \inf_{\bx}F(\bx)\right) - \gamma c_{t}\mE\left[\|\bx_{t} - \bx_{t_{0}}\|^{2}\right] \\
			& = -\gamma \mE\left[R_{t}(\bx_{t}) - \inf_{\bx}F(\bx)\right].
		\end{aligned}
	\end{equation}
	Hence 
	\begin{equation}
		\small
		\mE\left[F(\bx_{t_{1}}) - \inf_{\bx}F(\bx)\right] \leq e^{-\gamma(t_{1} - t_{0})}\mE\left[F(\bx_{t_{0}}) - \inf_{\bx}F(\bx)\right].
	\end{equation}
	Then for any $T = t_{m}$ and $\Delta = t_{i + 1} - t_{i} = \cdots = t_{1} - t_{0}$, we have 
	\begin{equation}
		\small
		\mE\left[F(\bx_{T}) - \inf_{\bx}F(\bx)\right] \leq e^{-\gamma T}\mE\left[F(\bx_{0}) - \inf_{\bx}F(\bx)\right],
	\end{equation}
	which implies the exponential convergence rate of SVRG flow under PL inequality. 
	\par
	On the other hand, let us check the computational complexity of it under PL inequality. As can be seen, to make $\mE[F(\bx_{t_{1}}) - \inf_{\bx}F(\bx)] \leq \epsilon$, we should take $I\Delta = \eta T = \cO(\gamma^{-1}\log{\epsilon^{-1}})$. Due to \eqref{eq:svrg pl delta t}, the SVRG flow should be conducted for $I = \max\{\cO(\sqrt{\eta}\gamma^{-1}\log{\epsilon^{-1}}), \cO(\log{\epsilon^{-1}})\}$ epochs with $M = \Delta / \eta$ steps in each epoch, and the computational complexity for each epoch is of order $\cO(M + N) = \cO(\Delta / \eta + N)$. So that the total computational complexity is of order 
	\begin{equation}
		\small
		I\cO(\Delta / \eta + N) = \max\{\cO(\sqrt{\eta}\gamma^{-1}\log{\epsilon^{-1}}), \cO(\log{\epsilon^{-1}})\}(\min\left\{\cO(1 / \sqrt{\eta}), \cO(1 / \gamma)\right\}/\eta + N),
	\end{equation} 
	If $\gamma^{-1} \geq N^{1/3}$, we take $\eta = \cO(N^{-2/3})$, the above equality becomes $\cO((N + \gamma^{-1}N^{2/3})\log{\epsilon^{-1}})$. On the other hand, when $\gamma^{-1} \leq N^{1/3}$, and $\eta = \cO(N^{-2/3})$, the above equality is also $\cO((N + \gamma^{-1}N^{2/3})\log{\epsilon^{-1}})$. 
\end{proof}
\par
Next, we prove the convergence rate of Riemannian SVRG flow as mentioned in main body of this paper. The following theorem is the formal statement of Theorem \ref{thm:convergence of svrg on manifold}. 
\begin{theorem}\label{thm:convergence of svrg on manifold general}
	Let $\pi_{t}$ in \eqref{eq:svrg flow on manifold}, $\Delta = t_{1} - t_{0} = \cdots = t_{I} - t_{I - 1} = \cO(1 / \sqrt{\eta})$, and $\eta T = I\Delta$ for $I$ epochs. Then, if Assumption \ref{ass:stochastic continuous} and $\mE_{\pi_{t_{i}, t}}[\tr(\nabla^{2}\log{(d\pi_{t}/d\mu)}\Sigma_{\rm SVRG})] \leq \lambda_{t}\mE_{\pi_{t_{i}, t}}[\tr(\Sigma_{\rm SVRG})]$ holds for any $t$ and $\lambda_{t}$ is a polynomial of $t$, then  
	\begin{equation}
		\small
		\begin{aligned}
			\frac{1}{\eta T}\sum\limits_{i = 1}^{I}\int_{t_{i}}^{t_{i + 1}}\left\|\grad D_{KL}(\pi_{t}\parallel \mu)\right\|^{2}_{\pi_{t}}dt \leq \frac{2D_{KL}(\pi_{0}\parallel \mu)}{\eta T}.
		\end{aligned}
	\end{equation}
	On the other hand, by taking $\eta = \cO(N^{-2/3})$, the computational complexity of Riemannian SVRG flow is of order $\cO(N^{2/3}/\epsilon)$ when $\min_{0\leq t \leq \eta T} \left\|\grad D_{KL}(\pi_{t}\parallel \mu)\right\|^{2}_{\pi_{t}} \leq \epsilon$.
	\par
	Furthermore, when $F(\pi) = D_{KL}(\pi \parallel \mu)$ satisfies log-Sobolev inequality \eqref{eq:log Sobolev inequality}, we have 
	\begin{equation}
		\small
		D_{KL}(\pi_{\eta T}\parallel \mu) \leq e^{-\gamma \eta T}D_{KL}(\pi_{0}\parallel \mu).
	\end{equation} 
	Besides that, it takes $\cO((N + \gamma^{-1}N^{2/3})\log{\epsilon^{-1}})$ computational complexity to make $D_{KL}(\pi_{\eta T}\parallel \mu)\leq \epsilon$. 
\end{theorem}

\begin{proof}
	Let us consider the Lyapunov function for union probability measure $\pi_{t_{0}, t}$ with $t_{0}\leq t \leq t_{1}$
	\begin{equation}
		\small
		R_{t}(\pi_{t_{0}, t}) = D_{KL}(\pi_{t}\parallel \mu) + \frac{c_{t}}{2}\int \left\|\by - \bx\right\|^{2}\pi_{t_{0}, t}(\by, \bx) = D_{KL}(\pi_{t}\parallel \mu) + \frac{c_{t}}{2}G(\pi_{t_{0}, t}).
	\end{equation}
	Then 
	\begin{equation}\label{eq:svrg manifold drt}
		\small
		\begin{aligned}
			\frac{\partial}{\partial{t}}R_{t}(\pi_{t_{0}, t}) & = \left\langle\grad D_{KL}(\pi_{t}\parallel \mu), \bh(\pi_{t})\right\rangle_{\pi_{t}} + \frac{c_{t}^{\prime}}{2}G(\pi_{t_{0}, t}) + \frac{c_{t}}{2}\frac{\partial}{\partial{t}}G(\pi_{t_{0}, t}), 
		\end{aligned}
	\end{equation}
	where $\partial{\pi_{t}}/\partial{t} = \nabla\cdot\left(\pi_{t}\bh(\pi_{t})\right)$ is used to simplify the notations. Then due to \eqref{eq:svrg flow on manifold}, Lemma \ref{lem:trace inequality} and similar induction to \eqref{eq:svrg ode f descent}
	\begin{equation}\label{eq:svrg manifold descent}
		\small
		\begin{aligned}
			\left\langle\grad D_{KL}(\pi_{t}\parallel \mu), \bh(\pi_{t})\right\rangle_{\pi_{t}}& = -\left\|\grad D_{KL}(\pi_{t}\parallel \mu)\right\|_{\pi_{t}}^{2} \\
			& + \int \left\langle\nabla \log{\frac{d\pi_{t}}{d\mu}}(\bx), \frac{\eta}{2}\int \pi_{t_{0}, t}(\by, \bx)\nabla_{\bx}\cdot \Sigma_{\rm SVRG}(\by, \bx)\right\rangle d\by d\bx\\
			& + \int \left\langle \nabla \log{\frac{d\pi_{t}}{d\mu}}(\bx), \frac{\eta}{2}\int\pi_{t_{0}, t}(\by, \bx)\Sigma_{\rm SVRG}(\by, \bx)\nabla_{\bx}\log{\pi_{t_{0}, t}(\by, \bx)}\right\rangle d\by d\bx \\
			& = -\left\|\grad D_{KL}(\pi_{t}\parallel \mu)\right\|_{\pi_{t}}^{2} + \frac{\eta}{2}\mE_{\pi_{t_{0}, t}}\left[\tr\left(\nabla^{2}\log{\frac{d\pi_{t}}{d\mu}}(\bx)\Sigma_{\rm SVRG}(\by, \bx)\right)\right] \\
			& \leq -\left\|\grad D_{KL}(\pi_{t}\parallel \mu)\right\|_{\pi_{t}}^{2} + \frac{\eta\lambda_{t}}{2}\mE\left[\left\|\nabla \log{\mu_{\xi}}(\bx_{t_{0}}) - \nabla\log{\mu_{\xi}(\bx_{t})}\right\|^{2}\right] \\
			& \leq -\left\|\grad D_{KL}(\pi_{t}\parallel \mu)\right\|_{\pi_{t}}^{2} + \frac{\eta L_{2} \lambda_{t}}{2}G(\pi_{t_{0}, t}), 
		\end{aligned}
	\end{equation}
	where the last inequality is due to the Lipschitz continuity of $\nabla \log{\mu_{\xi}(\bx)}$. Similarly, by Fokker-Planck equation, we get 
	\begin{equation}\label{eq:svrg manifold partial G}
		\small
		\begin{aligned}
			\frac{\partial}{\partial{t}}G(\pi_{t_{0}, t}) & = \frac{\partial}{\partial{t}}\mE\left[\|\bx_{t} - \bx_{t_{0}}\|^{2}\right] \\
			& = \mE\left[\left\langle\bx_{t} - \bx_{t_{0}}, \frac{d\bx_{t}}{dt}\right\rangle\right] \\
			& = \mE\left[\left\langle\bx_{t} - \bx_{t_{0}}, \bh(\pi_{t}(\bx_{t}))\right\rangle\right] \\
			& \leq \frac{1}{2\beta}\|\grad D_{KL}(\pi_{t}\parallel \mu)\|_{\pi_{t}}^{2} + \frac{\beta}{2}G(\pi_{t_{0}, t}) + \frac{\eta}{2}\mE_{\pi_{t_{0}, t}}\left[\tr\left(\Sigma_{\rm SVRG}(\bx_{t_{0}}, \bx_{t})\right)\right] \\
			& \leq \frac{1}{2\beta}\|\grad D_{KL}(\pi_{t}\parallel \mu)\|_{\pi_{t}}^{2} + \left(\frac{\beta}{2} + \frac{\eta L_{2}}{2}\right)G(\pi_{t_{0}, t}).
		\end{aligned}
	\end{equation}
	Combining \eqref{eq:svrg manifold drt}, \eqref{eq:svrg manifold partial G} and \eqref{eq:svrg manifold descent}, we get 
	\begin{equation}
		\small
		\frac{\partial}{\partial{t}}R_{t}(\pi_{t_{0}, t}) \leq -\left(1 - \frac{c_{t}}{2\beta}\right)\left\|\grad D_{KL}(\pi_{t}\parallel \mu)\right\|^{2}_{\pi_{t}} + \left[\frac{(\beta + \eta L_{2})c_{t}}{2} + \frac{\eta\lambda_{t} L_{2}}{2} + \frac{c_{t}^{\prime}}{2}\right]G(\pi_{t_{0}, t}).
	\end{equation}
	By taking 
	\begin{equation}
		\small
		(\beta + \eta L_{2})c_{t} + \eta\lambda_{t} L_{2} + c_{t}^{\prime} = 0,
	\end{equation}
	which implies 
	\begin{equation}\label{eq:ct1}
		\small
		c_{t_{1}} = e^{-(\beta + \eta L_{2})(t_{1} - t_{0})}c_{t_{0}} - \int_{t_{0}}^{t_{1}}\eta L_{2}\lambda_{t}e^{-(\beta + \eta L_{2})(t - t_{0})}dt.
	\end{equation}
	Without loss of generality, let  
	\begin{equation}
		\small
		\begin{aligned}
			a_{p} & = \eta L_{2}\int_{t_{0}}^{t_{1}}(t - t_{0})^{p}e^{-(\beta + \eta L_{2})(t - t_{0})}dt \\
			& = -\frac{\eta L_{2}}{\beta + \eta L_{2}}(t_{1} - t_{0})^{p}e^{-(\beta + \eta L_{2})(t - t_{0})}dt + \eta L_{2}\int_{t_{0}}^{t_{1}}p(t - t_{0})^{p - 1}e^{-(\beta + \eta L_{2})(t - t_{0})}dt \\
			& = pa_{p - 1} -\frac{\eta L_{2}}{\beta + \eta L_{2}}(t_{1} - t_{0})^{p}e^{-(\beta + \eta L_{2})(t - t_{0})} \\
			& = \cdots \\
			& = -\frac{\eta L_{2}}{\beta + \eta L_{2}}\left(1 - e^{-(\beta + \eta L_{2})(t_{1} - t_{0})}\mathrm{Poly}(t_{1} - t_{0}, p)\right), 
		\end{aligned}
	\end{equation}
	where $\mathrm{Poly}(t_{1} - t_{0}, p)$ is a $p$-th order polynomial of $t_{1} - t_{0}$. W.l.o.g, let  
	\begin{equation}
		\small
		\lambda_{t} = \lambda_{t_{0}} + \mathrm{Poly}(t - t_{0}, p) = \lambda_{t_{0}} + \sum_{i=1}^{p}b_{i}(t - t_{0})^{p}. 
	\end{equation}
	Then we have 
	\begin{equation}
		\small
		\begin{aligned}
			\int_{t_{0}}^{t_{1}}\eta L_{2}\lambda_{t}& e^{-(\beta + \eta L_{2})(t - t_{0})}dt = \lambda_{t_{0}}a_{0} + \sum_{i=1}^{p}a_{i}b_{i}\\
			& = \frac{\eta L_{2}\lambda_{t_{0}}}{\beta + \eta L_{2}}\left(1 - e^{-(\beta + \eta L_{2})(t_{1} - t_{0})}\right) - \frac{\eta L_{2}}{\beta + \eta L_{2}}\left(\sum_{i=1}^{p}b_{i} + e^{-(\beta + \eta L_{2})(t_{1} - t_{0})}\mathrm{Poly}(t_{1} - t_{0}, p)\right) \\
			& = \frac{\eta L_{2}}{\beta + \eta L_{2}}\left(\lambda_{t_{0}} - \sum_{i=1}^{p}b_{i} - \mathrm{Poly}(t_{1} - t_{0}, p)e^{-(\beta + \eta L_{2})(t_{1} - t_{0})}\right)
		\end{aligned}
	\end{equation}
	By invoking $c_{t_{0}} = \sqrt{\eta}, c_{t_{1}} = 0$, $\beta = \sqrt{\eta}$, and the above equality into \eqref{eq:ct1} we get
	\begin{equation}
		\small
		t_{1} - t_{0} = \frac{1}{\sqrt{\eta} + \eta L_{2}}\log{\frac{1 + \sqrt{\eta}L_{2} + L_{2}\mathrm{Poly}(t_{1} - t_{0}, p)}{L_{2}(\lambda_{0} - \sum_{i=1}^{p}b_{i})}},
	\end{equation}
	which implies $t_{1} - t_{0} = \cO(1 / \sqrt{\eta})$ as in \eqref{eq:svrg delta t} (note that the value of $b_{i}$ are automatically adjusted to make the above equality meaningful), by taking $\eta\to 0$. Thus, similar to \eqref{eq:descent lemma for svrg ode}, we get 
	\begin{equation}
		\small
		\frac{1}{2(t_{1} - t_{0})}\int_{t_{0}}^{t_{1}}\left\|\grad D_{KL}(\pi_{t}\parallel \mu)\right\|^{2}dt \leq \frac{D_{KL}(\pi_{t_{0}}\parallel \mu) - D_{KL}(\pi_{t_{1}}\parallel \mu)}{t_{1} - t_{0}}.
	\end{equation}
	Then the two conclusions are similarly obtained as in Theorem \ref{thm:convergence of svrg on euclidean} by taking $\eta = \cO(N^{-2/3})$. 
\end{proof}
Notably, the imposed ``proper'' condition on $\pi_{0}$ can be implied by the polynomial upper bound to the spectral norm of Hessian. This is because, from Lemma \ref{lem:trace inequality} and semi-positive definite property of $\Sigma_{\rm SVRG}$, we know 
\begin{equation}
	\small
	\mE_{\pi_{t_{i}, t}}\left[\nabla^{2}\log{(d\pi_{t}/d\mu)}\Sigma_{\rm SVRG}\right] \leq \mE_{\pi_{t_{i, t}}}\left[\lambda_{\max}\left(\nabla^{2}\log{(d\pi_{t}/d\mu)}\right)\tr(\Sigma_{\rm SVRG})\right].
\end{equation}
Then, we may take
\begin{equation}
	\small
	\lambda_{t} = \frac{\mE_{\pi_{t_{i, t}}}\left[\lambda_{\max}\left(\tr(\nabla^{2}\log{(d\pi_{t}/d\mu)}\right)\tr(\Sigma_{\rm SVRG})\right]}{\mE_{\pi_{t_{i, t}}}\left[\tr(\Sigma_{\rm SVRG})\right]}.
\end{equation}
Due to the formulation of Riemannian SVRG \eqref{eq:svrg flow on manifold}, the density $\pi_{t}$ only depends on $\pi_{0}$ and $\mu$. Therefore, under properly chosen $\pi_{0}$, the obtained $\pi_{t}$ can satisfy the imposed condition on the spectral norm. Besides that, we do not impose any restriction on the order of $\lambda_{t}$, while the polynomial function class can approximate any function so that it can be extremely large. Therefore, the imposed polynomial order of $\lambda_{t}$ can be easily satisfied. 

\section{More than Riemannian PL inequality}\label{sec:More than Riemannian PL inequality}
We observe that Theorems \ref{thm:convergence riemannian gd}, \ref{thm:convergence riemannian sgd}, and \ref{thm:convergence of svrg on manifold} demonstrate the nice log-Sobolev inequality \eqref{eq:log Sobolev inequality} improves the convergence rates into global ones. Therefore, we briefly discuss the condition in this section. As mentioned in Section \ref{sec:Preliminaries}, the log-Sobolev inequality is indeed the Riemannian PL inequality in Riemannian manifold. To see this, for a function $f$ defined on $\bbR^{d}$, the PL inequality is that for any global minima $\bx^{*}$ of it, we have  
\begin{equation}\label{eq:pl inequality}
	\small
	2\gamma(f(\bx) - f(\bx^{*})) \leq \|\nabla f(\bx)\|^{2}  \qquad \mathrm{PL}
\end{equation} 
holds for any $\bx$. The PL inequality indicates that all local minima are global minima. It is a nice property that guarantees the global convergence in Euclidean space \citep{karimi2016linear}. Naturally, we can generalize it into Wasserstein space. To this end, let $F(\pi) = D_{KL}(\pi\parallel \mu)$,  the sole global minima is $\pi = \mu$ with $F(\mu) = 0$. Then, in Wasserstein space, the PL inequality \eqref{eq:pl inequality} is generalized to 
\begin{equation}\label{eq:riemannian pl}
	\small
	2\gamma D_{KL}(\pi \parallel \mu) \leq \|\grad D_{KL}(\pi \parallel \mu)\|_{\pi}^{2}, \qquad \mathrm{Riemannian\ PL}. 
\end{equation}
which is log-Sobolev inequality \eqref{eq:log Sobolev inequality} due to \eqref{eq:riemannian gradient of KL}. Thus, our Theorems \ref{thm:convergence riemannian gd}, \ref{thm:convergence riemannian sgd}, and \ref{thm:convergence of svrg on manifold} indicate that Riemannian PL inequality guarantees the global convergence on manifold. 
\par
Moreover, if $f$ in \eqref{eq:pl inequality} has Lipschitz continuous gradient, the properties of quadratic growth (QG) and error bound (EB) are equivalent to the PL inequality \citep{karimi2016linear}
\begin{equation}\label{eq:QG}
	\small
	f(\bx) - f(\bx^{*}) \geq \frac{\gamma}{2} \|\bx - \bx^{*}\|^{2} \qquad \mathrm{QG},
\end{equation}
\begin{equation}
	\small
	\|\nabla f(\bx)\| \geq \gamma\|\bx - \bx^{*}\| \qquad \mathrm{EB}, 
\end{equation}
so that global convergence. In Wasserstein space, the two properties are generalized as  
\begin{equation}
	\small
	D_{KL}(\pi\parallel \mu) \geq \frac{\gamma}{2} \sW_{2}^{2}(\pi, \mu) \qquad \mathrm{Riemannian\ QG},
\end{equation}
\begin{equation}
	\small
	\left\|\nabla\log{\frac{d\pi}{d\mu}}\right\|_{\pi} \geq \gamma \sW_{2}(\pi, \mu) \qquad \mathrm{Riemannian\ EB}.
\end{equation}
Then, the relationship between the three properties in Wasserstein space is illustrated by the following proposition. This proposition is from \citep{otto2000generalization}, and we prove it to make this paper self-contained.  
\begin{restatable}{proposition}{plinequalityrelationship}[Otto-Vallani]\citep{otto2000generalization}
	Let $F(\pi)$ be $D_{KL}(\pi\parallel \mu)$, then Riemannian PL $\Rightarrow$ Riemannian QG, Riemannian PL $\Rightarrow$ Riemannian EB.
\end{restatable}
This proposition indicates that Riemannian PL inequality implies the other two conditions, but not vice-versa. This is because the Lipschitz continuity of Riemannian gradient i.e., $\|\grad D_{KL}(\pi \parallel \mu)\|_{\pi}^{2} \leq L\sW_{2}^{2}(\pi, \mu)$ does not hold as in Euclidean space \citep{villani2009optimal} (Wasserstein distance is weaker than the KL divergence). 
\begin{proof}
	If Riemannian PL $\Rightarrow$ Riemannian QG then Riemannian PL $\Rightarrow$ Riemannian EB is naturally proved. Next, we prove the first claim. Let $G(\pi) = \sqrt{F(\pi)}$, then by chain-rule and Riemannian PL inequality, 
	\begin{equation}
		\small
		\left\|\grad G(\pi)\right\|_{\pi}^{2} = \left\|\frac{\grad D_{KL}(\pi\parallel \mu)}{2D^{\frac{1}{2}}_{KL}(\pi\parallel \mu)}\right\|^{2}_{\pi} \geq \frac{\gamma}{2}. 
	\end{equation}
	Then let us consider 
	\begin{equation}
		\small
		\begin{dcases}
			\frac{\partial}{\partial{t}}\pi_{t} = \Exp_{\pi_{t}}[-\grad G(\pi_{t})] = \nabla\cdot (\pi_{t}\grad G(\pi_{t})), \\
			\pi_{0} = \pi
		\end{dcases}
	\end{equation}
	where the last equality can be similarly proved as in \eqref{pro:riemannian gradient}. Then 
	\begin{equation}
		\small
		\begin{aligned}
			G(\pi_{0}) - G(\pi_{T}) & = \int_{0}^{T} \frac{\partial}{\partial{t}}G(\pi_{t})dt \\
			& = \int_{0}^{T} \mathrm{Dev}G(\pi_{t})[-\grad G(\pi_{t})]dt \\	
			& = \int_{0}^{T} \left\|-\grad G(\pi_{t})\right\|^{2}_{\pi_{t}}dt  \\
			& \geq \int_{0}^{T} \frac{\gamma}{2}dt \\
			& = \frac{\gamma T}{2}. 
		\end{aligned}
	\end{equation}
	Due to this, and $G(\pi) \ge 0$, there exists some $T$ such that $\pi_{T} = \mu$. Since $\pi_{t}$ is a curvature satisfies $\pi_{0} = \pi$ and $\pi_{T} = \mu$, due to the definition of Wasserstein distance is geodesic distance, we have 
	\begin{equation}
		\small
		\sW_{2}(\pi, \mu) = \sW_{2}(\pi_{0}, \pi_{T}) \leq \int_{0}^{T} \|\grad G(\pi_{t})\|_{\pi_{t}} dt .
	\end{equation}
	Thus, 
	\begin{equation}
		\small
		G(\pi_{0}) = G(\pi_{0}) - G(\pi_{T}) = \int_{0}^{T} \left\|-\grad G(\pi_{t})\right\|^{2}_{\pi_{t}}dt \geq \sqrt{\frac{\gamma}{2}}\int_{0}^{T}\geq  \|\grad G(\pi_{t})\|_{\pi_{t}} dt \geq \sqrt{\frac{\gamma}{2}}\sW_{2}(\pi, \mu),
	\end{equation}
	which implies our conclusion. 
\end{proof}

%% file: langevin_on_manifold.bbl
\begin{thebibliography}{}

\bibitem[Absil et~al., 2009]{absil2009optimization}
Absil, P.-A., Mahony, R., and Sepulchre, R. (2009).
\newblock {\em Optimization algorithms on matrix manifolds}.
\newblock Princeton University Press.

\bibitem[Balasubramanian et~al., 2022]{balasubramanian2022towards}
Balasubramanian, K., Chewi, S., Erdogdu, M.~A., Salim, A., and Zhang, S.
  (2022).
\newblock Towards a theory of non-log-concave sampling: first-order
  stationarity guarantees for langevin monte carlo.
\newblock In {\em Conference on Learning Theory}.

\bibitem[Becigneul and Ganea, 2018]{becigneul2018riemannian}
Becigneul, G. and Ganea, O.-E. (2018).
\newblock Riemannian adaptive optimization methods.
\newblock In {\em International Conference on Learning Representations}.

\bibitem[Bonnabel, 2013]{bonnabel2013stochastic}
Bonnabel, S. (2013).
\newblock Stochastic gradient descent on riemannian manifolds.
\newblock {\em IEEE Transactions on Automatic Control}, 58(9):2217--2229.

\bibitem[Bottou et~al., 2018]{bottou2018optimization}
Bottou, L., Curtis, F.~E., and Nocedal, J. (2018).
\newblock Optimization methods for large-scale machine learning.
\newblock {\em SIAM review}, 60(2):223--311.

\bibitem[Boumal et~al., 2019]{boumal2019global}
Boumal, N., Absil, P.-A., and Cartis, C. (2019).
\newblock Global rates of convergence for nonconvex optimization on manifolds.
\newblock {\em IMA Journal of Numerical Analysis}, 39(1):1--33.

\bibitem[Chatterji et~al., 2018]{chatterji2018theory}
Chatterji, N., Flammarion, N., Ma, Y., Bartlett, P., and Jordan, M. (2018).
\newblock On the theory of variance reduction for stochastic gradient monte
  carlo.
\newblock In {\em International Conference on Machine Learning}.

\bibitem[Cheng and Bartlett, 2018]{cheng2018convergence}
Cheng, X. and Bartlett, P. (2018).
\newblock Convergence of langevin mcmc in kl-divergence.
\newblock In {\em Algorithmic Learning Theory}.

\bibitem[Chewi, 2023a]{chewi2023log}
Chewi, S. (2023a).
\newblock Log-concave sampling.
\newblock {\em Lecture Notes}.

\bibitem[Chewi, 2023b]{chewi2023optimization}
Chewi, S. (2023b).
\newblock {\em An optimization perspective on log-concave sampling and beyond}.
\newblock PhD thesis, Massachusetts Institute of Technology.

\bibitem[Chewi et~al., 2022]{chewi2022analysis}
Chewi, S., Erdogdu, M.~A., Li, M., Shen, R., and Zhang, S. (2022).
\newblock Analysis of langevin monte carlo from poincare to log-sobolev.
\newblock In {\em Conference on Learning Theory}.

\bibitem[Cho and Lee, 2017]{cho2017riemannian}
Cho, M. and Lee, J. (2017).
\newblock Riemannian approach to batch normalization.
\newblock In {\em Advances in Neural Information Processing Systems}.

\bibitem[Du et~al., 2019]{du2019gradient}
Du, S., Lee, J., Li, H., Wang, L., and Zhai, X. (2019).
\newblock Gradient descent finds global minima of deep neural networks.
\newblock In {\em International conference on machine learning}.

\bibitem[Du et~al., 2018]{du2018algorithmic}
Du, S.~S., Hu, W., and Lee, J.~D. (2018).
\newblock Algorithmic regularization in learning deep homogeneous models:
  Layers are automatically balanced.
\newblock {\em Advances in Neural Information Processing Systems}.

\bibitem[Dubey et~al., 2016]{dubey2016variance}
Dubey, K.~A., J~Reddi, S., Williamson, S.~A., Poczos, B., Smola, A.~J., and
  Xing, E.~P. (2016).
\newblock Variance reduction in stochastic gradient langevin dynamics.
\newblock {\em Advances in Neural Information Processing Systems}.

\bibitem[Durmus and Moulines, 2019]{durmus2019high}
Durmus, A. and Moulines, {\'E}. (2019).
\newblock High-dimensional bayesian inference via the unadjusted langevin
  algorithm.
\newblock {\em Bernoulli}, 25(4A):2854--2882.

\bibitem[Dwivedi et~al., 2018]{dwivedi2018log}
Dwivedi, R., Chen, Y., Wainwright, M.~J., and Yu, B. (2018).
\newblock Log-concave sampling: Metropolis-hastings algorithms are fast!
\newblock In {\em Conference on learning theory}.

\bibitem[Fang et~al., 2018]{fang2018spider}
Fang, C., Li, C.~J., Lin, Z., and Zhang, T. (2018).
\newblock Spider: Near-optimal non-convex optimization via stochastic
  path-integrated differential estimator.

\bibitem[Ghadimi and Lan, 2013]{ghadimi2013stochastic}
Ghadimi, S. and Lan, G. (2013).
\newblock Stochastic first-and zeroth-order methods for nonconvex stochastic
  programming.
\newblock {\em SIAM Journal on Optimization}, 23(4):2341--2368.

\bibitem[Hu et~al., 2019]{hu2019diffusion}
Hu, W., Li, C.~J., Li, L., and Liu, J.-G. (2019).
\newblock On the diffusion approximation of nonconvex stochastic gradient
  descent.
\newblock {\em Annals of Mathematical Sciences and Applications}, 4(1).

\bibitem[Johnson and Zhang, 2013]{johnson2013accelerating}
Johnson, R. and Zhang, T. (2013).
\newblock Accelerating stochastic gradient descent using predictive variance
  reduction.
\newblock In {\em Advances in neural information processing systems}, pages
  315--323.

\bibitem[Karatzas and Shreve, 2012]{karatzas2012brownian}
Karatzas, I. and Shreve, S. (2012).
\newblock {\em Brownian motion and stochastic calculus}, volume 113.
\newblock Springer Science \& Business Media.

\bibitem[Karimi et~al., 2016]{karimi2016linear}
Karimi, H., Nutini, J., and Schmidt, M. (2016).
\newblock Linear convergence of gradient and proximal-gradient methods under
  the polyak-{\l}ojasiewicz condition.
\newblock In {\em Machine Learning and Knowledge Discovery in Databases:
  European Conference}.

\bibitem[Kinoshita and Suzuki, 2022]{kinoshita2022improved}
Kinoshita, Y. and Suzuki, T. (2022).
\newblock Improved convergence rate of stochastic gradient langevin dynamics
  with variance reduction and its application to optimization.

\bibitem[Li et~al., 2017]{li2017stochastic}
Li, Q., Tai, C., and Weinan, E. (2017).
\newblock Stochastic modified equations and adaptive stochastic gradient
  algorithms.
\newblock In {\em International Conference on Machine Learning}.

\bibitem[Liu et~al., 2019]{liu2019understanding}
Liu, C., Zhuo, J., Cheng, P., Zhang, R., and Zhu, J. (2019).
\newblock Understanding and accelerating particle-based variational inference.
\newblock In {\em International Conference on Machine Learning}.

\bibitem[Liu, 2017]{liu2017stein}
Liu, Q. (2017).
\newblock Stein variational gradient descent as gradient flow.

\bibitem[Liu et~al., 2017]{liu2017accelerated}
Liu, Y., Shang, F., Cheng, J., Cheng, H., and Jiao, L. (2017).
\newblock Accelerated first-order methods for geodesically convex optimization
  on riemannian manifolds.
\newblock In {\em Advances in Neural Information Processing Systems}.

\bibitem[Maoutsa et~al., 2020]{maoutsa2020interacting}
Maoutsa, D., Reich, S., and Opper, M. (2020).
\newblock Interacting particle solutions of fokker--planck equations through
  gradient--log--density estimation.
\newblock {\em Entropy}, 22(8):802.

\bibitem[Mousavi-Hosseini et~al., 2023]{mousavi2023towards}
Mousavi-Hosseini, A., Farghly, T., He, Y., Balasubramanian, K., and Erdogdu,
  M.~A. (2023).
\newblock Towards a complete analysis of langevin monte carlo: Beyond
  poincar$\backslash$'e inequality.
\newblock In {\em Conference on learning theory}.

\bibitem[Nguyen et~al., 2017]{nguyen2017sarah}
Nguyen, L.~M., Liu, J., Scheinberg, K., and Tak{\'a}{\v{c}}, M. (2017).
\newblock Sarah: A novel method for machine learning problems using stochastic
  recursive gradient.
\newblock In {\em International Conference on Machine Learning}.

\bibitem[Oksendal, 2013]{oksendal2013stochastic}
Oksendal, B. (2013).
\newblock {\em Stochastic differential equations: an introduction with
  applications}.
\newblock Springer Science \& Business Media.

\bibitem[Orvieto and Lucchi, 2019]{orvieto2019continuous}
Orvieto, A. and Lucchi, A. (2019).
\newblock Continuous-time models for stochastic optimization algorithms.
\newblock {\em Advances in Neural Information Processing Systems}.

\bibitem[Otto and Villani, 2000]{otto2000generalization}
Otto, F. and Villani, C. (2000).
\newblock Generalization of an inequality by talagrand and links with the
  logarithmic sobolev inequality.
\newblock {\em Journal of Functional Analysis}, 173(2):361--400.

\bibitem[Parisi, 1981]{parisi1981correlation}
Parisi, G. (1981).
\newblock Correlation functions and computer simulations.
\newblock {\em Nuclear Physics B}, 180(3):378--384.

\bibitem[Petersen et~al., 2008]{petersen2008matrix}
Petersen, K.~B., Pedersen, M.~S., et~al. (2008).
\newblock The matrix cookbook.
\newblock {\em Technical University of Denmark}, 7(15):510.

\bibitem[Reddi et~al., 2016]{reddi2016stochastic}
Reddi, S.~J., Hefny, A., Sra, S., Poczos, B., and Smola, A. (2016).
\newblock Stochastic variance reduction for nonconvex optimization.
\newblock In {\em International Conference on Machine Learning}.

\bibitem[Ren and Wang, 2024]{ren2024ornstein}
Ren, P. and Wang, F.-Y. (2024).
\newblock Ornstein- uhlenbeck type processes on wasserstein spaces.
\newblock {\em Stochastic Processes and their Applications}, 172:104339.

\bibitem[Sam~Patterson, 2013]{patterson2013stochastic}
Sam~Patterson, Y. W.~T. (2013).
\newblock Stochastic gradient riemannian langevin dynamics on the probability
  simplex.

\bibitem[Santambrogio, 2017]{santambrogio2017euclidean}
Santambrogio, F. (2017).
\newblock $\{$Euclidean, metric, and Wasserstein$\}$ gradient flows: an
  overview.
\newblock {\em Bulletin of Mathematical Sciences}, 7:87--154.

\bibitem[Shiryaev, 2016]{shiryaev2016probability}
Shiryaev, A.~N. (2016).
\newblock {\em Probability-1}, volume~95.
\newblock Springer.

\bibitem[Song et~al., 2020]{song2020score}
Song, Y., Sohl-Dickstein, J., Kingma, D.~P., Kumar, A., Ermon, S., and Poole,
  B. (2020).
\newblock Score-based generative modeling through stochastic differential
  equations.
\newblock In {\em International Conference on Learning Representations}.

\bibitem[Su et~al., 2014]{su2014differential}
Su, W., Boyd, S., and Candes, E. (2014).
\newblock A differential equation for modeling nesterov’s accelerated
  gradient method: theory and insights.

\bibitem[Van~Handel, 2014]{van2014probability}
Van~Handel, R. (2014).
\newblock Probability in high dimension.
\newblock {\em Lecture Notes (Princeton University)}.

\bibitem[Villani et~al., 2009]{villani2009optimal}
Villani, C. et~al. (2009).
\newblock {\em Optimal transport: old and new}, volume 338.
\newblock Springer.

\bibitem[Von~Neumann, 1937]{von1937some}
Von~Neumann, J. (1937).
\newblock {\em Some matrix-inequalities and metrization of matric space}.

\bibitem[Wang and Li, 2020]{wang2020information}
Wang, Y. and Li, W. (2020).
\newblock Information newton's flow: second-order optimization method in
  probability space.
\newblock Preprint arXiv:2001.04341.

\bibitem[Wang and Li, 2022]{wang2022accelerated}
Wang, Y. and Li, W. (2022).
\newblock Accelerated information gradient flow.
\newblock {\em Journal of Scientific Computing}, 90:1--47.

\bibitem[Welling and Teh, 2011]{welling2011bayesian}
Welling, M. and Teh, Y.~W. (2011).
\newblock Bayesian learning via stochastic gradient langevin dynamics.
\newblock In {\em International Conference on Machine Learning}.

\bibitem[Yi, 2022]{yi2022accelerating}
Yi, M. (2022).
\newblock Accelerating training of batch normalization: A manifold perspective.
\newblock In {\em Uncertainty in Artificial Intelligence}.

\bibitem[Yi et~al., 2022]{yi2022characterization}
Yi, M., Wang, R., and Ma, Z.-M. (2022).
\newblock Characterization of excess risk for locally strongly convex
  population risk.

\bibitem[Zhang et~al., 2016]{zhang2016riemannian}
Zhang, H., J~Reddi, S., and Sra, S. (2016).
\newblock Riemannian svrg: Fast stochastic optimization on riemannian
  manifolds.

\bibitem[Zhang and Sra, 2016]{zhang2016first}
Zhang, H. and Sra, S. (2016).
\newblock First-order methods for geodesically convex optimization.
\newblock In {\em Conference on Learning Theory}, pages 1617--1638.

\bibitem[Zhang and Sra, 2018]{zhang2018towards}
Zhang, H. and Sra, S. (2018).
\newblock Towards riemannian accelerated gradient methods.
\newblock Preprint arXiv:1806.02812.

\bibitem[Zhang et~al., 2018]{zhang2018r}
Zhang, J., Zhang, H., and Sra, S. (2018).
\newblock R-spider: A fast riemannian stochastic optimization algorithm with
  curvature independent rate.
\newblock Preprint arXiv:1811.04194.

\bibitem[Zou et~al., 2018]{zou2018subsampled}
Zou, D., Xu, P., and Gu, Q. (2018).
\newblock Subsampled stochastic variance-reduced gradient langevin dynamics.
\newblock In {\em International Conference on Uncertainty in Artificial
  Intelligence}.

\bibitem[Zou et~al., 2019]{zou2019sampling}
Zou, D., Xu, P., and Gu, Q. (2019).
\newblock Sampling from non-log-concave distributions via variance-reduced
  gradient langevin dynamics.
\newblock In {\em International Conference on Artificial Intelligence and
  Statistics}.

\end{thebibliography}
